\documentclass{article}
\pdfoutput=1

\usepackage[nonatbib,final]{neurips_2022}

\usepackage[utf8]{inputenc} 
\usepackage[T1]{fontenc}    
\usepackage{hyperref}       
\usepackage{url}            

\usepackage{booktabs}       
\usepackage{amsfonts}       

\usepackage{amsmath}
\usepackage{amssymb}
\usepackage{mathtools}
\usepackage{amsthm}

\usepackage{enumitem}
\usepackage{bbm}

\usepackage{nicefrac}       
\usepackage{microtype}      
\usepackage{xcolor}         
\usepackage{pifont}
\newcommand{\cmark}{\ding{51}}%
\newcommand{\xmark}{\ding{55}}%

\theoremstyle{plain}
\newtheorem{theorem}{Theorem}
\newtheorem{proposition}[theorem]{Proposition}
\newtheorem{lemma}[theorem]{Lemma}

\theoremstyle{definition}

\newtheorem{assumption}[theorem]{Assumption}
\theoremstyle{remark}
\newtheorem{remark}[theorem]{Remark}

\DeclareMathOperator*{\argmin}{arg\,min}

\providecommand\S{}
\renewcommand{\S}{\mathcal{S}}
\providecommand\A{}
\renewcommand{\A}{\mathcal{A}}
\providecommand\E{}
\renewcommand{\E}{\mathbb{E}}

\usepackage{etoolbox}
\newcommand{\zerodisplayskips}{%
\setlength{\abovedisplayskip}{0pt}%
\setlength{\abovedisplayshortskip}{0pt}%
\setlength{\belowdisplayskip}{0pt}%
\setlength{\belowdisplayshortskip}{0pt}%
}
\appto{\normalsize}{\zerodisplayskips}
\appto{\small}{\zerodisplayskips}
\appto{\footnotesize}{\zerodisplayskips}

\usepackage{algorithm}
\usepackage{algorithmic}

\usepackage{subfigure}
\usepackage{tikz}
\usetikzlibrary{positioning}
\usetikzlibrary{shapes.multipart}
\usetikzlibrary{arrows.meta}

\title{DOPE: Doubly Optimistic and Pessimistic Exploration for Safe
Reinforcement Learning}

\author{Archana Bura, ~Aria HasanzadeZonuzy,  ~Dileep Kalathil,\\ \quad \textbf{Srinivas Shakkottai},~\textbf{Jean-Francois Chamberland} \\
Department of Electrical and Computer Engineering,  Texas A\&M University\\
\texttt{\{archanabura,azonuzy,dileep.kalathil,sshakkot,chmbrlnd\}@tamu.edu}
}

\begin{document}
\maketitle
\begin{abstract}
Safe reinforcement learning is extremely challenging--not only must the agent explore an unknown environment, it must do so while ensuring no safety constraint violations. We formulate this safe  reinforcement learning (RL) problem using the framework of a finite-horizon Constrained Markov Decision Process (CMDP) with an unknown transition probability function, where we model the safety requirements as constraints on the expected cumulative costs that must be satisfied during all episodes of learning.  We propose a model-based safe RL algorithm that we call Doubly Optimistic and Pessimistic Exploration (DOPE), and show that it achieves an objective regret $\tilde{O}(|\mathcal{S}|\sqrt{|\mathcal{A}| K})$ without violating the safety constraints during learning, where  $|\mathcal{S}|$ is the number of states, $|\mathcal{A}|$ is the number of actions, and $K$ is the number of learning episodes.  Our key idea is to combine a reward bonus for exploration (optimism) with a conservative constraint (pessimism), in addition to the standard optimistic model-based exploration.  DOPE is not only able to improve the objective regret bound, but also shows a significant empirical performance improvement as compared to earlier optimism-pessimism approaches. 
\end{abstract}
\section{Introduction}\label{sec:Intro}
  
    

Constrained Markov Decision Processes (CMDPs) impose 
restrictions 
that pertain to resource or safety constraints of the system.  For example, the average radiated power in a wireless communication system must be restricted due to user health and battery lifetime considerations or 
the frequency of braking or accelerating in an autonomous vehicle must be kept bounded to ensure passenger comfort.  Since these systems have complex dynamics, a constrained reinforcement learning (CRL) approach is attractive for determining an optimal control policy.   But how do we ensure that safety or resource availability constraints are not violated while learning such an optimal control policy?


Our goal is to develop a framework for safe exploration without constraint violation (with high probability) for solving CMDP problems where the model is unknown.  While there has been much work on RL for both the MDP and the CMDP setting, ensuring \emph{safe exploration} in the CRL setting has received less attention.  The problem is challenging, since we do not allow constraint violation either during learning or deployment, while ensuring a low regret in terms of the optimal objective.  Our aim is to explore a model-based approach in an episodic setting under which the model (the transition kernel of the CMDP) is empirically determined as samples from the system are gathered.

\begin{table*}[t]
\caption{Comparison of safe RL algorithms for CMDPs. The algorithms are: OptCMDP, OptCMDP-Bonus \cite{efroni2020exploration},  AlwaysSafe \cite{simao2021alwayssafe} and OptPessLP \cite{liu2021learning}. This table is presented for $K \geq \text{poly}(|S|, |A|, H)$, with polynomial terms independent of $K$ omitted. Expected regret results are in high probability.}
\label{tab:compare}
\vskip 0.15in
   \begin{center}
\begin{scriptsize}
\begin{sc}
\begin{tabular}{lcccccr}
\toprule
Algorithm & Model & Reward  & Constraint & Objective & Constraint & Empirical\\
~ & Optimism & Optimism & Pessimism & Regret & Regret & Perf.\\
\midrule
OptCMDP & \cmark & \xmark & \xmark & $\tilde{O}( |\mathcal{S}| \sqrt{|\mathcal{A}|H^4K})$ & $\tilde{O}( |\mathcal{S}| \sqrt{|\mathcal{A}|H^4K})$ & - \\
OptCMDP-B & \xmark & \cmark & \xmark & $\tilde{O}( |\mathcal{S}| \sqrt{|\mathcal{A}|H^4K})$ & $\tilde{O}( |\mathcal{S}| \sqrt{|\mathcal{A}|H^4K})$ & -\\
AlwaysSafe & \cmark & \xmark & Heuristic 
& Unknown
& 0   & \xmark     \\
OptPess-LP & \xmark & \cmark & \cmark & $\tilde{O}(\frac{H^3}{\bar{C}-\bar{C}_b} \sqrt{|\mathcal{S}|^3 |\mathcal{A}|K})$ & 0  & \xmark      \\
\midrule
DOPE & \cmark & \cmark & \cmark & $\tilde{O}(\frac{H^3}{\bar{C}-\bar{C}_b} |\mathcal{S}| \sqrt{|\mathcal{A}|K})$ & $0$ & \cmark\\
\bottomrule
\end{tabular}
\end{sc}
\end{scriptsize}
\end{center}
    \vspace{-0.3in}
\end{table*}

There has been much recent interest in model-based RL approaches to solving a constrained MDP, the most relevant of which we have summarized in Table~\ref{tab:compare}.  The setup is of a finite horizon episodic CMDP, with a state space of size $|\mathcal{S}|,$ an action space of size $|\mathcal{A}|$ and a horizon of length $H.$  Regret is measured over the first $K$ episodes of the algorithm.  Both the attained objective and constraint satisfaction are computed in an expected sense for a given policy. Allowing constraint violations during learning means that the algorithm suffers both an objective regret and a constraint regret.  Examples of algorithms following this method are OptCMDP, OptCMDP-Bonus \cite{efroni2020exploration}.  The algorithms use different ways of incentivizing exploration to obtain samples to build up an empirical system model. OptCMDP uses the idea of optimism in the face of uncertainty from the model perspective and solves an extended linear program to find the best model (highest reward) under the samples gathered thus far. OptCMDP-Bonus uses a different approach of adding a bonus to the reward term to incentivize exploration of state-action pairs that have fewer samples, which we consider as a form of optimism from the reward perspective.  While the objective regret of both these algorithms is $\tilde{O}(\sqrt{K}),$  the constraint regret too is of the same order, since constraints may be violated during learning.

An algorithm that begins with no knowledge of the model will have exploration steps during learning that might violate constraints in expectation.   Hence, safe RL approaches assume the availability of an inexpert baseline policy that does not violate the constraints, but is insufficient to explore the entire state action space.  So it cannot be simply applied until a high-accuracy model is obtained.  A heuristic approach entitled AlwaysSafe \cite{simao2021alwayssafe} assumes a factored CMDP that allows the easy generation of a safe baseline policy.   It then combines the optimism with respect to the model of OptCMDP with a heuristically chosen hardening of constraints. 
 This guarantees no constraint violations, but does not have regret guarantee with respect to the objective.  Its empirical performance is variable and its use is limited to factored CMDP problems.

The OptPessLP algorithm \cite{liu2021learning} formalizes the idea of coupling optimism and pessimism by starting with OptCMDPBonus that has an optimistic reward, and systematically applying decreasing levels of pessimism with respect to the constraint violations.   The approach is successful in ensuring the twin goals of a $\tilde{O}(\sqrt{K})$ objective regret, while ensuring no constraint violations.   However, the authors do not present any empirical performance evaluation results.  When we implemented OptPessLP, we found that the performance is singularly bad in that linear regret persists for a large number of samples, and the tapering off to $\tilde{O}(\sqrt{K})$ regret behavior does not appear to happen quickly.  The problem with this algorithm is that it is so pessimistic with regard to constraints that it heavily disincentivizes exploration and ends up choosing the base policy for long sequences.

The issue upon which algorithm performance depends is the choice of how to combine optimism and pessimism to obtain both order optimal and empirically good objective regret performance, while ensuring no constraint violations happen.  Our insight is that optimism with respect to the model is a key enabler of exploration, and can be coupled with the addition of optimism with respect to the reward.  This \emph{double dose of optimism}---both with respect to model and reward---could ensure that pessimistic hardening of constraints does not excessively retard exploration.  Following this insight, we develop DOPE, a doubly optimistic and pessimistic exploration approach. 
We are able to show that DOPE not only attains $\tilde{O}(\sqrt{K})$ objective regret behavior with zero constraint regret with high probability, it also reduces the objective regret bound over OptPessLP by a factor of $\sqrt{|\mathcal{S}|}.$  We conduct performance analysis simulations under representative CMDP problems and show that DOPE easily outperforms all the earlier approaches.  Thus, the idea of double optimism is not only valuable from the order optimal algorithm design perspective, it also shows good empirical regret performance, indicating the feasibility of utilizing the methodology in real-world systems. 
The code for the experiments in this paper is located at: \url{https://github.com/archanabura/DOPE-DoublyOptimisticPessimisticExploration}

\textbf{Related Work:}
\textbf{Constrained RL:} Constrained Markov Decision Processes (CMDP) has been an active area of research \cite{altman1999constrained}, with applications in domains such as power systems \cite{singh2018decentralized, li2019constrained}, communication networks \cite{altman2002applications, singh2018throughput}, and robotics \cite{ding2013strategic, chow2015trading}.  In  \cite{borkar2005actor}, the author proposed  an actor-critic RL  algorithm for learning the asymptotically optimal policy for an infinite horizon average cost CMDP when the model is unknown. 
This approach is also utilized in function approximation settings with asymptotic local convergence guarantees \cite{bhatnagar2012online, chow2017risk, tessler2018reward}.  Policy gradient algorithms for CMDPs have also been developed \cite{achiam2017constrained, yang2019projection, zhang2020first, paternain2019constrained, ding2020natural, liu2021fast, zhou2022anchor}.  However, none these works address the problem of safe exploration to provide guarantees on the constraint violations during learning.

\noindent \textbf{Safe  Multi-Armed Bandits:} 
The problem of  safe exploration in linear bandits with  stage-wise safety constraint is studied in  \cite{amani2019linear, khezeli2020safe, moradipari2020stage, pacchiano2021stochastic}. Linear bandits with more general constraints have also been studied \cite{parulekar2020stochastic, liu2021efficient}.
These works do not consider the more challenging RL setting which involves an underlying dynamical system with unknown model.

\noindent \textbf{Safe  Online Convex Optimization:} 
Online convex optimization \cite{hazan2016introduction} has been studied with stochastic constraints \cite{yu2017online, chaudhary2022safe}  and adversarial constraints \cite{neely2017online, yuan2018online, liakopoulos2019cautious}. These allow constraint violation during learning and characterize the cumulative amount of violation. 
A safe Frank-Wolf algorithm for convex optimization with unknown linear stochastic constraints has been studied in \cite{usmanova2019safe}. However, these too do not consider the RL setting with an unknown model.

\noindent \textbf{Exploration in Constrained RL:}
There has been much work in this area \textit{with} constraint violations during learning, including the work discussed in the introduction~\cite{efroni2020exploration}.   
These include \cite{singh2020learning,hasanzadezonuzy2021learning,kalagarla2021sample}, which derive bounds either on the objective and constraint regret or on the sample complexity of learning an $\epsilon$-optimal policy. Other works on safe RL include ~\cite{wei2022triple, wei2022provably}, where a model-free approach is considered, and~\cite{AmaniSafeOfflineRL}, that pertains to offline RL. These are complementary to our model-based approach. The problem of learning the optimal policy of a CMDP without violating the constraints was also studied in  \cite{zheng2020constrained}. However, they assume that the model is known and only the cost functions are unknown, whereas we address more difficult problem with unknown model and cost functions. 

\textbf{Notations:} For any integer $M$, $[M]$ denotes the set $\{1, \ldots, M\}$. For any two real numbers $a, b$,  $a \vee b := \max\{a, b\}$. For any given set $\mathcal{X}$,  $\Delta(\mathcal{X})$ denotes the probability simplex over the set $\mathcal{X}$, and $|\mathcal{X}|$ denotes the cardinality of the set $\mathcal{X}$.

\section{Preliminaries and Problem Formulation}

\subsection{Constrained Markov Decision Process} 

We address the safe exploration problem using the framework of  episodic Constrained Markov Decision Process (CMDP) \cite{altman1999constrained}. We consider  a CMDP, denoted as $M = \langle \S, \A, r,c, P, H,  \bar{C} \rangle$ with $r = (r_{h})^{H}_{h=1}, c = (c_{h})^{H}_{h=1}, P = (P_{h})^{H}_{h=1}$, where $\S$ is the state space, $\A$ is the action space,   $H$  is the episode length,  $r_h: \S \times \A \rightarrow \mathbb{R}$ is the objective cost function at time step $h \in [H]$, $c_h: \S \times \A \rightarrow \mathbb{R}$ is the constraint cost function at time step $h \in [H]$, $P_{h}$ is the transition probability function with $P_h(s'|s,a)$ representing the probability of transitioning to state $s'$ when action $a$ is taken at state $s$ at time $h$. In the RL context, the transition matrix $P$ is also called the model of the CMDP.  Finally, $\bar{C}$ is a scalar that specifies the safety constraint in terms of the maximum permissible value for the expected cumulative constraint cost. We consider the setting where $|\S|$ and $|\A|$ are finite.  Also, without loss of generality, we assume that costs $r$ and $c$ are bounded in $[0,1]$.

A non-stationary randomized policy  $\pi = (\pi_{h})^{H}_{h=1}, \pi_{h} : \S \rightarrow \Delta(\mathcal{A})$ specifies the control action to be taken at each time step $h \in [H]$. In particular, $\pi_{h}(s, a)$ denotes the probability of taking action $a$ when the state is $s$ at time step $h$. For
an arbitrary cost function  $l: [H] \times \S \times \A \rightarrow \mathbb{R}$,   the  value function of a policy $\pi$ corresponding to time step $h \in [H]$ given a state  $s \in \S$ is defined as
\begin{align}
\label{eq:value-function-defn}
\begin{aligned}
&V_{l,h}^{\pi}(s;P) = \E [ \sum_{\tau=h}^H l_{\tau}(s_{\tau},a_{\tau}) | s_{h} =s ],
\end{aligned}
\end{align}
where $ a_{\tau} \sim \pi_{\tau}(s_{\tau}, \cdot),~s_{\tau+1} \sim P_{\tau}(\cdot|s_{\tau}, a_{\tau})$. 
Since we are mainly interested  in the value of a policy starting from $h=1$, we simply denote $V_{l,1}^{\pi}(s;P)$ as  $V_l^{\pi}(s;P)$. For the rest of the paper, we will assume that the initial state is $s_{1}$ is fixed.  So, we will simply denote $V_l^{\pi}(s_{1};P)$ as $V_l^{\pi}(P)$, when it is clear from the context. This standard assumption \cite{efroni2020exploration, ding2021provably}  can be made without loss of generality.

The CMDP (planning) problem with a known model $P$ can then be stated as follows:
\begin{align}
   \label{eq:CMDP-OPT} 
    \quad \min_{\pi}~~~V_r^{\pi}(P) \quad \text{s.t.} \quad V_c^{\pi}(P) \leq \bar{C}.
\end{align}
We say that a policy $\pi$ is a \textbf{safe  policy} if $V_c^{\pi}(P) \leq \bar{C}$, i.e., if the expected cumulative constraint cost corresponding to the policy $\pi$ is less than the maximum permissible value $\bar{C}$. The \textbf{set of safe policies}, denoted as   $\Pi_{\text{safe}}$, is defined  as $\Pi_{\text{safe}} = \{\pi : V_c^{\pi}(P) \leq \bar{C}\}$. Without loss of generality, we assume that the  CMDP problem \eqref{eq:CMDP-OPT} is feasible, i.e., $\Pi_{\text{safe}}$ is non-empty. Let $\pi^{*}$ be the \textbf{optimal safe policy}, which is the solution of \eqref{eq:CMDP-OPT}. 

The CMDP (planning) problem is significantly different from the standard Markov Decision Process (MDP) (planning) problem \cite{altman1999constrained}. Firstly, there may not exist an optimal deterministic policy for a CMDP, whereas the existence of  a deterministic optimal policy is well known for a standard MDP. Secondly, there does not exist a Bellman optimality principle or Bellman equation for CMDP. So, the standard dynamic programming solution approaches   which rely on the Bellman equation   cannot be directly applied to solve the CMDP problem.  

There are two standard approaches for solving the CMDP problem, namely the Lagrangian approach and the linear programming (LP)  approach. Both approaches exploit the zero duality gap property of the CMDP problem \cite{altman1999constrained} to find the optimal policy. In this work, we will use the LP approach, consistent with model optimism. Details of solving \eqref{eq:CMDP-OPT} using the LP approach are in Appendix \ref{sec:lp-cmdp}.

\subsection{Reinforcement Learning with Safe Exploration}

The goal of the reinforcement learning with  safe exploration is  to solve ~\eqref{eq:CMDP-OPT}, but without the knowledge of the model $P$ a priori. Hence, the learning algorithm has to perform exploration by employing different policies to learn $P$. However,  we also want the exploration for learning to have a safety guarantee,  i.e, the policies employed during learning should belong to the set of safe policies $\Pi_{\text{safe}}$. Since  $\Pi_{\text{safe}}$ itself is defined based on the  unknown $P$, the learning algorithm will not know $\Pi_{\text{safe}}$ a priori. This makes the safe exploration problem extremely challenging. 

We consider a model-based RL algorithm that  interacts with the environment in an episodic manner.  Let $\pi_{k} = (\pi_{h,k})^{H}_{h=1}$ be the policy employed by the algorithm in episode $k$. At each time step $h \in [H]$ in an episode $k$, the algorithm observes  state $s_{h,k}$, selects  action $a_{h,k} \sim \pi_{h,k}(s_{h,k}, \cdot)$, and incurs the costs $r_h(s_{h,k}, a_{h,k})$ and $c_h(s_{h,k}, a_{h,k})$. The next state $s_{h+1,k}$ is realized according to the probability vector $P(\cdot | s_{h,k}, a_{h,k})$. As stated before, for simplicity, we assume that the initial state  is fixed for each episode $k \in [K] := \{1, \ldots, K\}$, i.e., $s_{1,k} = s_{1}$. We also assume that the maximum permissible cost  $\bar{C}$ for any exploration policy is known and it is specified as part of the learning problem. 

The performance of the RL algorithm is measured using the metric of \textit{safe objective regret}. The safe objective regret is defined exactly as the standard regret of an RL  algorithm for  exploration in MDPs \cite{jaksch2010near, dann2017unifying, azar2017minimax}, but with an additional constraint that the exploration polices should  belong to the safe set  $\Pi_{\text{safe}}$. Formally, the safe objective regret $R(K)$ after $K$ learning episodes is defined as  
\begin{equation}
\label{eq:regret}
\begin{split}
    &R(K) = \sum_{k=1}^K (V^{\pi_k}_{r}(P) - V^{\pi^*}_{r}(P)), ~   \pi_{k} \in  \Pi_{\text{safe}}, \forall k \in [K].
\end{split}
\end{equation}
Since $\Pi_{\text{safe}}$ is unknown, clearly it is not possible to employ a safe policy without making any additional assumptions. We overcome this  obvious limitation by assuming that the algorithm has access to a \textbf{safe  baseline policy} $\pi_{b}$ such that $\pi_{b} \in \Pi_{\text{safe}}$. We formalize this assumption as follows. 
\begin{assumption}[Safe baseline policy]
\label{assp:baseline}
The  algorithm knows a safe baseline policy $\pi_b$ such that  $V_c^{\pi_{b}}(P) = \bar{C}_b$, where $\bar{C}_b < \bar{C}$. The value  $\bar{C}_b$ is also known to the algorithm. 
\end{assumption}
\begin{remark}
Knowing a safe policy $\pi_{b}$ is necessary for solving the safe RL problem because we require the constraint to always be satisfied.  A similar assumption has been used in the  case of safe exploration in linear bandits \cite{amani2019linear, khezeli2020safe, pacchiano2021stochastic}, as well as in earlier work on safe RL \cite{simao2021alwayssafe,liu2021learning}.  
\end{remark}

\section{Algorithm and Performance Guarantee}
\label{sec:algorithm}
DOPE builds on the  \textit{optimism in the face of uncertainty (OFU)} style exploration algorithms for RL  \cite{jaksch2010near, dann2017unifying}, using such optimism, both in terms of the model, as well as to provide a reward bonus for under-explored state-actions.   However, a naive OFU-style algorithm may lead to selecting exploration policies that are not in the safe set $\Pi_{\text{safe}}$. So we modify the selection of exploratory policy by incorporating \textit{pessimism in the face of uncertainty (PFU)}  on the constraints, making DOPE doubly optimistic and pessimistic in exploration.

DOPE operates in episodes, each of length $H$. Define the  filtration $\mathcal{F}_{k}$  as the sigma algebra generated by the observations until the end of episode $k \in [K]$, i.e., $\mathcal{F}_{k} = (s_{h,k'}, a_{h,k'}, h \in [H], k' \in [k])$. Let $n_{h,k}(s,a) = \sum^{k-1}_{k'=1} \mathbbm{1}\{s_{h,k'} = s, a_{h,k'}  = a\} $ be the number of times the pair $(s,a)$ was observed at time step $h$ until the beginning of episode $k$. Similarly, define  $n_{h,k}(s,a, s') = \sum^{k-1}_{k'=1}  \mathbbm{1}\{s_{h,k'} = s, a_{h,k'}  = a, s_{h+1,k'} = s'\} $.  At the beginning of each episode $k$, DOPE estimates the model as $\widehat{P}_{h,k}(s'|s, a) = {n_{h,k}(s,a, s')}/({n_{h,k}(s,a) \vee 1})$. Similar to OFU-style algorithms, we construct a confidence set  $\mathcal{P}_{k}$   around $\widehat{P}_{k}$ as  $\mathcal{P}_{k} = \cap_{(s, a) \in \S \times \A} \mathcal{P}_{k}(s, a)$, where
\begin{align}
\label{eq:confidenceset}
&\mathcal{P}_{k}(s, a) =  \{{P'} :   | {P'_h}(s'|s,a) -\widehat{P}_{h,k}(s'|s,a)|
\leq \beta^p_{h,k}(s,a,s'), \forall h \in [H], s'\in\S\},\\
\label{eq:beta}
&\beta^p_{h,k}(s,a,s') = \sqrt{\frac{4 \text{Var}(\widehat{P}_{h,k}(s'|s,a)) L}{n_{h,k}(s,a)  \vee 1}}+ \frac{14 L}{3 (n_{h,k}(s,a)\vee 1)},
\end{align}
where $L = \log(\frac{2SAHK}{\delta})$, and  $\text{Var}(\widehat{P}_{h,k}(s'|s,a)) = \widehat{P}_{h,k}(s'|s,a) (1-\widehat{P}_{h,k}(s'|s,a))$.  Using the empirical Bernstein inequality, we can show that the true model $P$ is an element of $\mathcal{P}_{k}$ for any $k \in [K]$ with probability at least $1-2 \delta$ (see Appendix \ref{sec:useful-lemmas}).

Similarly, at the beginning of each episode $k$, DOPE estimates the unknown objective and constraint costs as $\hat{r}_{h,k}(s,a) =  \frac{\sum_{k'=1}^{k-1} r_h(s,a) \mathbbm{1}\{s_{h,k'} = s, a_{h,k'} = a\}}{n_{h,k}(s,a) \vee 1}$, $\hat{c}_{h,k}(s,a) = \frac{\sum_{k'=1}^{k-1} c_h(s,a) \mathbbm{1}\{s_{h,k'} = s, a_{h,k'} = a\}}{n_{h,k}(s,a) \vee 1}$.
In keeping with OFU, we construct confidence sets $\mathcal{R}_k$ and $\mathcal{C}_k$ around $\hat{r}_k$ and $\hat{c}_k$ respectively, as 
\begin{align}\label{eqn: rewconfidencesets}
\begin{aligned}
&\mathcal{R}_{k} = \{\tilde{r} : |\tilde{r}_h(s,a) -\hat{r}_{h,k}(s,a)| \leq \beta^l_{h,k}(s,a), \forall h,s,a \in [H] \times \S \times \A \},\\
&\mathcal{C}_{k} = \{\tilde{c} : |\tilde{c}_h(s,a) -\hat{c}_{h,k}(s,a)| \leq \beta^l_{h,k}(s,a), \forall h,s,a \in [H] \times \S \times \A \},\\
&\beta^l_{h,k}(s,a) =  \sqrt{L'/(n_h^k(s,a) \vee 1)},
\end{aligned}
\end{align}

where $L' = 2 \log({6SAHK}/{\delta})$, and $\tilde{r} = (\tilde{r}_{h})^{H}_{h=1}, \tilde{c} = (\tilde{c}_{h})^{H}_{h=1}$. Using Hoeffding inequality, we can show that the true costs belong to $\mathcal{R}_k$ and $\mathcal{C}_k$ for any $k \in [K]$ with probability at least $1-\delta$ {(see Appendix C)}.
We define $\mathcal{M}_k = \mathcal{P}_k \cap \mathcal{R}_k \cap \mathcal{C}_k$ to be the total confidence ball.  

It is tempting to use the  standard OFU approach for selecting the exploration polices since this approach is known to provide sharp regret guarantees for exploration problems in RL. The standard OFU approach will find the optimistic model $\underline{P}_{k}$ and optimistic policy $\underline{\pi}_{k}$, where 
\begin{align}
\label{eq:ofu-optimization}
  (\underline{\pi}_{k}, \underline{P}_{k}) = \argmin_{{\pi'}, (P',r',c') \in \mathcal{M}^{k}}~~  V_{r'}^{\pi'}(P'), ~~ \text{s.t.} ~~   V_{c'}^{\pi'}(P') \leq \bar{C}.
\end{align}

The OFU problem \eqref{eq:ofu-optimization} is feasible  since the true model $M$ is an element of $\mathcal{M}_{k}$  (with high probability). In particular, $(\pi_{b}, P)$ and $(\pi^{*}, P)$ are feasible solutions of \eqref{eq:ofu-optimization}.  Moreover, \eqref{eq:ofu-optimization} can be solved efficiently using an extended linear programming approach, as described in Appendix \ref{sec:extended-lp-cmdp}. The policy $\underline{\pi}_{k}$ ensures exploration while satisfying the constraint  $V_c^{\underline{\pi}_{k}}(\underline{P}_{k}) \leq \bar{C}$. However, this naive OFU approach overlooks the important issue that  $\underline{\pi}_{k}$ may not be a safe policy with respect to the true model $P$. More precisely,  it is possible to have  $V_c^{\underline{\pi}_{k}}(P) > \bar{C}$  even though $V_c^{\underline{\pi}_{k}}(\underline{P}_{k}) \leq \bar{C}$.  So, the standard  OFU approach alone will not give a safe exploration strategy. 

In order to ensure that the exploration policy employed at any episode is safe, we add a pessimistic penalty to the empirical constraint cost to get the pessimistic constraint cost function $\bar{c}_{k}$ as
\begin{align}
    \label{eqn:pessimistic-cost}
    \bar{c}_{h,k}(s,a) = \hat{c}_{h,k}(s,a) + \beta^l_{h,k}(s,a) + H \bar{\beta}^p_{h,k}(s,a),
\end{align}
where $\bar{\beta}^p_{h,k}(s,a) = \sum_{s' \in \S}\beta^p_{h,k}(s,a,s').$
Since $\bar{\beta}^p_{h,k}(s,a)$ is $\tilde{O}({1}/{\sqrt{n_{h,k}(s, a)}})$, $(s,a)$ pairs that are less observed have a higher penalty, disincentivizing their exploration. However. such a pessimistic penalty may prevent the  exploration that is necessary to learn the optimal policy.  To overcome this issue, we  also modify the empirical objective cost function by subtracting a term to incentivize exploration, to obtain an optimistic objective cost function 
\begin{align}
 \label{eqn:optimistic-objective}
 \begin{aligned}
\bar{r}_{h,k}(s,a) &= \hat{r}_{h,k}(s,a) - \frac{3H}{\bar{C}-\bar{C}_b}\beta^l_{h,k}(s,a)  -\frac{H^2}{\bar{C}-\bar{C}_b} \bar{\beta}_{h,k}^p(s,a).
\end{aligned}
\end{align}
Since $\bar{\beta}^p_{h,k}(s,a)$ is $\tilde{O}({1}/{\sqrt{n_{h,k}(s, a)}})$, $(s,a)$ pairs that are less observed will have a lowered cost to incentivize their exploration. 

We select the policy ${\pi}_{k}$  for episode $k$ by solving the Doubly Optimistic-Pessimistic (DOP) problem: 
\begin{equation}
\label{eq:DOPE-optimization}
{(\pi}_{k}, {P}_{k}) = \argmin_{{\pi'}, {P'} \in \mathcal{P}_{k}}~~  V_{\bar{r}_k}^{\pi'}(P')  ~~\text{s.t.} ~~   V_{\bar{c}_{k}}^{\pi'}(P') \leq \bar{C}.
\end{equation}

{Notice that DOPE is doubly optimistic by considering both the optimistic objective cost function in~\eqref{eqn:optimistic-objective} and the optimistic model $P_k$ from the confidence set $\mathcal{P}_k$ in~\eqref{eq:DOPE-optimization}, while being pessimistic on the constraint in~\eqref{eq:DOPE-optimization}. Later, in Lemma~\ref{lem:termII-decomp-Optimism} in the appendix, we prove that $(\pi_k,P_k)$ is indeed an optimistic solution. This is in contrast with~\cite{liu2021learning}, where the optimism is solely limited to the objective cost.} 
We will show that our approach carefully balances double optimism and pessimism, yielding a regret minimizing learning algorithm with episodic safe exploration guarantees.

We note that  the DOP problem \eqref{eq:DOPE-optimization} may not be feasible, especially in the first few episodes of learning. This is because,  $\bar{\beta}^p_{h,k}(s,a)$ and $\beta_{h,k}^l(s,a)$ may be  large   during the initial phase of learning so that there may not be a   policy  $\pi'$ and a model $P' \in \mathcal{P}_{k}$  that can satisfy the constraint   $V_{\bar{c}_{k}}^{\pi'}(P') \leq \bar{C}$. We overcome this issue by employing a safe baseline policy $\pi_{b}$ (as defined in Assumption \ref{assp:baseline}) in the first $K_{o}$ episodes, a value provided by Proposition~\ref{prop:Feasibility}.  Since $\pi_{b}$ is safe by definition, DOPE ensures safety during the first $K_{o}$ episodes.  We will later show that the DOP problem  \eqref{eq:DOPE-optimization}  will have a  feasible solution after the first $K_{o}$ episodes (see Proposition \ref{prop:Feasibility}). 
For any episode $k \geq K_{o}$, DOPE employs policy $\pi_{k}$, which is the solution of \eqref{eq:DOPE-optimization}. We will also show that ${\pi}_{k}$ from  \eqref{eq:DOPE-optimization} (once it becomes feasible) will indeed be a safe policy (see Proposition \ref{prop:safety}).  We present DOPE  formally in Algorithm \ref{algo:DOPE}. 
\begin{algorithm}
\caption{Doubly Optimistic and Pessimistic Exploration (DOPE)}
\label{algo:DOPE}
\begin{algorithmic}[1]
        \STATE \textbf{Input:} $\delta \in \left(0,1\right), r, c, \pi_{b}, \bar{C}_{b}, \bar{C}, K_{o}$ 
        \STATE \textbf{Initialization:}  $n_{h,k}(s, a) = n_{h,k}(s, a, s') = 0 ~~ \forall s, s' \in S, a \in A, h \in [H].$ 
        \FOR {episodes $k = 1,\ldots, K$}
        \STATE Compute the estimates $\widehat{P}_{k}$, $\hat{r}_k$, $\hat{c}_k$ and the confidence set $\mathcal{M}_{k}$ according to \eqref{eq:confidenceset} - \eqref{eq:beta}, and \eqref{eqn: rewconfidencesets}.
         \IF{$k \leq K_{o}$}
         \STATE Select the exploration policy $\pi_{k} = \pi_{b}$
         \ELSE 
         \STATE Select the exploration policy $\pi_{k}$ according to \eqref{eq:DOPE-optimization}
         \ENDIF
         \FOR{$h = 1,2,\ldots,H$}
         \STATE Observe  state $s_{h,k}$, select action $a_{h,k}  \sim \pi_{h,k}(s_{h,k}, \cdot)$,  incur the  cost $r_h(s_{h,k}, a_{h,k})$ and  $c_h(s_{h,k}, a_{h,k})$, and observe next state $s_{h+1,k} \sim P_h(\cdot|s_{h,k}, a_{h,k})$ 
         \STATE Update the counts: $n_{h,k}(s_{h,k}, a_{h,k}) \leftarrow n_{h,k}(s_{h,k}, a_{h,k}) + 1$, $n_{h,k}(s_{h,k}, a_{h,k}, s_{h+1,k}) \leftarrow n_{h,k}(s_{h,k}, a_{h,k}, s_{h+1,k}) + 1$
        \ENDFOR
        \ENDFOR
        \end{algorithmic}
\end{algorithm}
\vspace{-0.3cm}

We now present our main result, which shows that the DOPE algorithm achieves $\tilde{O}(\sqrt{K})$ regret \textit{without} violating the safety constraints during learning, with high probability.

\begin{theorem}
\label{thm:main-regret-theorem}
Fix any $\delta \in (0, 1)$. Consider the  DOPE algorithm with $K_{o}$ as specified in Proposition \ref{prop:Feasibility}. Let $\{\pi_{k}, k \in [K]\}$ be the sequence of policies generated by the DOPE algorithm. Then, with probability at least $1-5\delta$,  $\pi_{k} \in \Pi_{\textnormal{safe}}$ for all $k \in [K]$. Moreover, with probability at least $1-5\delta$,  the regret of the DOPE algorithm satisfies
 \begin{equation*}
 R(K) \leq \tilde{\mathcal{O}}(  \frac{SH^3}{(\bar{C} - \bar{C}_{b})} \sqrt{A K}).
 \end{equation*}
\end{theorem}

\vspace{-0.1in}
\section{Analysis}
\label{sec:analysis}

We now provide the technical analysis of DOPE, concluding with the proof outline of Theorem \ref{thm:main-regret-theorem}.

\subsection{Preliminaries}
For an arbitrary policy $\pi'$ and transition probability function $P'$, define $\epsilon^{\pi'}_{k}(P')$ and $\eta^{\pi'}_{k}(P')$ as
\begin{align}
\label{eq:epsilon-k-pi-P}
\epsilon^{\pi'}_{k}(P')=H \mathbb{E}[\sum^{H}_{h=1}  \bar{\beta}^p_{h,k}(s_{h,k}, a_{h,k}) | \pi', P', \mathcal{F}_{k-1}],
\eta^{\pi'}_{k}(P') = \mathbb{E}[\sum^{H}_{h=1}  \beta^l_{h,k}(s_{h,k}, a_{h,k}) | \pi', P', \mathcal{F}_{k-1}].
\end{align}

Then, it is straightforward to show that (see \eqref{eq:Vk-decomp-1} - \eqref{eq:Vk-decomp-2} in the Appendix)
$
V^{\pi'}_{\bar{c}_{k}}(P') = V^{\pi'}_{\hat{c}_k}(P') + \eta_k^{\pi'}(P) + \epsilon^{\pi'}_{k}(P'),
V^{\pi'}_{\bar{r}_{k}}(P') = V^{\pi'}_{\hat{r}_k}(P') - \frac{3H}{\bar{C}-\bar{C}_b} \eta^{\pi'}_k(P') - \frac{H}{\bar{C}-\bar{C}_b} \epsilon^{\pi'}_{k}(P'). 
$
The analysis utilizes this decomposition of $V^{\pi'}_{\bar{c}_{k}}(P')$ and $V^{\pi'}_{\bar{r}_{k}}(P'),$ and the properties of $\epsilon^{\pi'}_{k}(P')$ and $\eta^{\pi'}_k(P')$. Under the good event $G$ defined as in Lemma~\ref{lem:Good-event}, we can show, $V_{\hat{c}_k}^{\pi'}(P') - \eta_k^{\pi'}(P') \leq V_c^{\pi'}(P') \leq V_{\hat{c}_k}^{\pi'}(P') + \eta_k^{\pi'}(P')$.

\subsection{Feasibility of the OP Problem}
Even though $(\pi_{b}, P)$ is a feasible solution to the original CMDP problem \eqref{eq:CMDP-OPT}, it may not be a feasible for the DOP problem \eqref{eq:DOPE-optimization} in the initial phase of learning. To see this, note that  $V^{\pi_{b}}_{\bar{c}_{k}}(P) =V^{\pi_{b}}_{\hat{c}_k}(P) + \eta^{\pi_b}_k(P) + \epsilon^{\pi_{b}}_{k}(P),$ and since $V_c^{\pi_b}(P) \ge V_{\hat{c}_k}^{\pi_b}(P) - \eta_k^{\pi_b}(P) $ under the good event, and $V^{\pi_{b}}_{{c}}(P) =  \bar{C}_{b}$,  we will have  $V^{\pi_{b}}_{\bar{c}_{k}}(P) \leq \bar{C}$ if  $2\eta^{\pi_{b}}_{k}(P) + \epsilon^{\pi_{b}}_{k}(P) \leq (\bar{C} - \bar{C}_{b})$.   So, $(\pi_{b}, P)$ is  a feasible solution for  \eqref{eq:DOPE-optimization} if  $2\eta^{\pi_{b}}_{k}(P) + \epsilon^{\pi_{b}}_{k}(P) \leq (\bar{C} - \bar{C}_{b})$. This sufficient condition may not be satisfied for initial episodes. However, since  $\epsilon^{\pi_{b}}_{k}(P)$ and $\eta^{\pi_b}_k(P)$
are decreasing in $k$, if  $(\pi_{b}, P)$  becomes a feasible solution for  \eqref{eq:DOPE-optimization} at episode $k'$, then it will remain feasible for all episodes $k \geq k'$.  Also, since $\bar{\beta}^p_{h,k}$ and $\beta^l_{h,k}$ decrease with $k$, one can expect that  \eqref{eq:DOPE-optimization} becomes feasible after some  number of episodes. We use these intuitions, along with some technical lemmas to show the following result. 

\begin{proposition}
\label{prop:Feasibility}
Under the DOPE algorithm, with a probability greater that $1-5 \delta$, $(\pi_{b}, P)$ is a feasible solution for the DOP  problem  \eqref{eq:DOPE-optimization}  for all $k \geq K_{o}$, where $K_{o} = \tilde{\mathcal{O}}( \frac{S^{2} A H^{4}}{ (\bar{C} - \bar{C}_{b})^{2}})$.
 \end{proposition}

\vspace{-0.5cm}
\subsection{Safety Exploration Guarantee}
We  show that the DOPE algorithm provides a safe exploration guarantee, i.e., $\pi_{k} \in \Pi_{\text{safe}}$ for all $k \in [K]$ with high probability, where $\pi_{k}$ is the exploration policy   employed by  DOPE in episode $k$. This is achieved by the carefully designed pessimistic constraint of the DOP problem \eqref{eq:DOPE-optimization}.  

For any   $k \leq K_{o}$, since $\pi_{k} = \pi_{b}$, and it is safe by Assumption \ref{assp:baseline}. For $k \geq K_{o}$, \eqref{eq:DOPE-optimization} is feasible according to Proposition \ref{prop:Feasibility}.  Since  $({\pi}_{k}, {P}_k)$ is the solution of \eqref{eq:DOPE-optimization}, we have  $V_{\bar{c}_k}^{{\pi}_{k}}({P}_k) =  V^{{\pi}_{k}}_{\hat{c}_k}({P}^{k}) +\eta^{ {\pi}_{k}}_{k}({P}_{k}) + \epsilon^{ {\pi}_{k}}_{k}({P}_{k}) \leq \bar{C}$. This implies that $V^{ {\pi}_{k}}_{\hat{c}_k}({P}_{k}) + \eta^{{\pi}_{k}}_{k}({P}_{k})\leq \bar{C}  - \epsilon^{{\pi}_{k}}_{k}({P}_{k})$.  
We have that $V_c^{\pi_k}(P_k) \leq V_{\hat{c}^k}^{\pi_k}(P_k) + \eta_k^{\pi_k}(P_k)$ under the good event, and hence, the above equation implies that $V_c^{\pi_k}(P_k)  \leq \bar{C}  - \epsilon^{{\pi}_{k}}_{k}({P}_{k}) $,  i.e., ${\pi}_{k}$ satisfies a tighter constraint with respect to the model $P_{k}$.  However, it is not obvious that the policy ${\pi}_{k}$ will be safe with respect to the true model $P$ because   $V^{{\pi}_{k}}_{c}(P)$ may be larger than $V^{{\pi}_{k}}_{c}({P}_{k})$ due to the change from ${P}_{k}$ to $P$. 

We, however, show that $V^{{\pi}_{k}}_{c}(P)$ cannot be larger than  $V^{{\pi}_{k}}_{c}({P}_{k})$ by more than $\epsilon^{{\pi}_{k}}_{k}({P}_{k})$, i.e.,  $V^{{\pi}_{k}}_{c}(P) -V^{{\pi}_{k}}_{c}({P}_{k}) \leq \epsilon^{\pi_{k}}_{k}({P}_{k})$. This will then   yield that $V^{{\pi}_{k}}_{c}(P) \leq V^{{\pi}_{k}}_{c}({P}_{k}) + \epsilon^{ {\pi}_{k}}_{k}({P}_{k}) \leq \bar{C}$, which is the true safety constraint. The key idea is in the  design of the pessimistic cost function $\bar{c}_{k}(\cdot, \cdot)$  such that its pessimistic  effect will balance the change in the value function (from $V^{{\pi}_{k}}_{c}({P}_{k})$ to $V^{{\pi}_{k}}_{c}(P)$)  due to the optimistic selection of the model ${P}_{k}$. We formally state the safety guarantee of DOPE below. 
\begin{proposition}
\label{prop:safety}
Let $\{\pi_{k}, k \in [K]\}$ be the sequence of policies generated by the DOPE algorithm. Then $\pi_{k}$ is safe  $\forall k \in [K]$, i.e., $V^{\pi_{k}}_{c}(P) \leq \bar{C}$, for all $k \in [K]$, with a probability greater than $1 - 5 \delta$.
\end{proposition}

\vspace{-0.4cm}
\subsection{Regret Analysis}

The regret analysis for most OFU style RL algorithms follows the standard approach of decomposing the regret into two terms as $R_{k} = V^{\pi_k}_{r}(P) - V^{\pi^*}_{r}(P) = ( V^{\pi_k}_{r}(P) - V^{\pi_k}_{r}({P}_k) ) + (V^{\pi_k}_{r}({P}_k) - V^{\pi^*}_{r}(P))$, where $R_{k}$ denotes the regret in episode $k$. The first term is the difference between value functions of the selected policy $\pi_{k}$ with respect to the true model $P$ and optimistic  model ${P}_{k}$. Bounding this term is the key technical analysis part of most of the OFU algorithms for the unconstrained MDP \cite{jaksch2010near, dann2015sample} and also the CMDP \cite{efroni2020exploration}.  In the standard OFU style analysis for the unconstrained problem, since $P \in \mathcal{P}_{k}$ for all $k$, it can be easily observed that  $(\pi^{*}, P)$ is a feasible solution for the OFU problem \eqref{eq:ofu-optimization} for all $k \in [K]$. Moreover, since  $(\pi_{k}, {P}_{k})$ is the optimal solution in $k^{th}$ episode, we get $V^{\pi_k}_{r}({P}_k) \leq V^{\pi^*}_{r}(P)$. So, the second term  will be non-positive, and hence can be dropped from the regret analysis. However, in our setting, the second term can be positive since  $(\pi^{*}, P)$ may not be a feasible solution of the DOP problem \eqref{eq:DOPE-optimization} due to the  pessimistic constraint. This necessitates a different approach for bounding the regret. 
{Existing work~\cite{liu2021learning} only considers optimism in the objective cost, and hence their proof closely follows that of OptCMDP-Bonus algorithm in~\cite{efroni2020exploration} with pessimistic constraints. In analyzing the regret of DOPE, we need to handle the optimism in objective cost as well as the model in regret terms, along with the pessimistic constraints. This make the analysis particularly  challenging.} The full proof is detailed in the appendix. 
\section{Experiments}\label{sec:MainExperiments}
We now evaluate DOPE via experiments.  We have two relevant metrics, namely, (i) objective regret, defined in~\eqref{eq:regret} that measures the optimality gap of the algorithm, and (ii) constraint regret, defined as $\sum_{k=1}^K \max \{0, V^{\pi_k}_c(P) - \bar{C} \},$ where $\pi_k$ is the output of the algorithm in question at episode $k.$  This measures the safety gap of the algorithm.  Our candidate algorithms are (i) OptCMDP,  (ii) AlwaysSafe, (iii) OptPessLP and (iv) DOPE, all described in the introduction.   OptCMDP is expected to show constraint violations, while the other three should show zero constraint regret.  We consider two environments here, with a third environment presented in the appendix.  AlwaysSafe can directly be used only with a factored CMDP, and only applies to the first environment presented.  We simulate both variants of this algorithm, referred to as AlwaysSafe~$\pi_T$ and AlwaysSafe~$\pi_\alpha$, respectively~\cite{simao2021alwayssafe}.

\textbf{Factored CMDP:}
We first consider a CMDP where the safety relavant features of the model can be separated, as shown in~\cite{simao2021alwayssafe}.  This CMDP has states $\{1,2,3\}$ arranged in a circle, and $2$ actions $\{1,2\}$ in each state,  to move right or stay put, respectively. The transitions move the agent to its right state with probability $1$, if action $1$ is taken. If action $2$ is taken, it remains in the same state with probability $1$. Action $1$ does not incur any objective cost or constraint cost. Action $2$ incurs an objective cost equals to the state number, and a constraint cost of $1$. We choose episode length $H=6$, and constraint as $\bar{C} = 3$.   The structure of this CMDP allows AlwaysSafe to extract a safe baseline policy from it. 

\textbf{Media Streaming CMDP:}  Our second environment represents media streaming to a device from a wireless base station, which provides high and low service rates at different costs.  These service rates have independent Bernoulli distributions, with parameters $\mu_1 = 0.9$, and $\mu_2 = 0.1$, where $\mu_1$ corresponds to the fast service.  Packets received at the device are stored in a media buffer and played out according to a Bernoulli process with parameter $\gamma.$  We denote the number of incoming packets into the buffer as $A_h$, and the number of packets leaving the buffer $B_h.$  The media buffer length is the state and evolves as 
$s_{h+1} = \min\{\max(0, s_h + A_h - B_h),N\},$
where $N=20$ is the maximum buffer length.  The action space is $\{1,2\}$, where action $1$ corresponds to using the fast service.
The objective cost is $r(s,a) = \mathbbm{1}\{s=0\},$ while the constraint cost is $c(s,a) = \mathbbm{1}\{a=1\}$, i.e., we desire to minimize the outage cost, while limiting the usage of fast service. We consider episode length $H=10$, and constraint $\bar{C} = \frac{H}{2}$.

\textbf{Experiment Setup:} 
OptCMDP and OptPessLP algorithms have not been implemented earlier.  For accurate comparison, we simulate all the algorithms true to their original formulations of cost functions and confidence intervals.  Our experiments are carried out for $20$ random runs, and averaged to obtain the regret plots. For DOPE, we choose $K_0$ to be as specified in Proposition~\ref{prop:Feasibility}.
 Full details on the algorithm parameters and experiment settings are provided in Appendix~\ref{sec:experiments}. 

\textbf{Baseline Policies:} Both OptPessLP and DOPE require baseline policies. We select the baseline policies as the optimal solutions of the given CMDP with a constraint $\bar{C}_b = 0.2 \bar{C}$. We choose the same baseline policies for both the algorithms.
This choice is to showcase the efficacy of DOPE, despite starting with a conservative baseline policy.

\textbf{Results for Factored CMDP:}  Fig.~\ref{fig:objective-regretsimple} shows the objective regret of the algorithms in this environment.  OptCMDP has a good objective regret performance as expected, but shows constraint violations.  AlwasySafe fails to achieve $\tilde{O}(\sqrt{K})$ regret in the episodes shown for both variants, although the $\pi_\alpha$ variant has smaller regret as compared to the $\pi_T$ variant.  OptPessLP takes a long time to attain $\sqrt{K}$ behavior, which means that is chooses the baseline policy for an extended period, and shows high empirical regret.  This suggests that reward optimism of OptPessLP is insufficient to balance the pessimism in the constraint.    DOPE not only achieves the desired $\tilde{O}(\sqrt{K})$ regret, but also does so in fewer episodes compared to the other two algorithms. Furthermore, it has low empirical regret.
Fig.~\ref{fig:constraint-regretsimple} shows that the constraint violation regret is zero for all the episodes of learning for all the safe algorithms, while OptCMDP shows a large constraint violation regret.

\textbf{Results for Media Streaming CMDP:} Fig.\ref{fig:objective-regretmedia} 
compares the objective regret across algorithms.   The value of DOPE over OptPessLP is more apparent here. After a linear growth phase, the objective regret of DOPE changes to a square-root scaling.   OptPessLP has not explored sufficiently at this point, and hence suffers high linear regret. Finally, OptCMDP also has square-root regret scaling, but is fastest, since it is not constrained by safe exploration.  Fig.\ref{fig:constraint-regretMedia} compares the regret in constraint violation for these algorithms.  As expected, DOPE and OptPessLP do not violate the constraint, while the OptCMDP algorithm significantly violates the constraints during learning.

\textbf{Effect of Baseline Policy:} We compare the objective regret of DOPE under different baseline policies in Fig.\ref{fig:objective-regretsimplecb} and Fig.\ref{fig:obj-regretMediacb}.  We see that less conservative (but safe) baselines result in lower regret, but the difference is not excessive, implying that the exact baseline policy chosen is not crucial.

\textbf{Summary}: DOPE has two valuable properties: $(i)$ Faster rate of shift to $\sqrt{K}$ behavior, since the linear regret phase where it applies the baseline policy is relatively short, and $(ii)$ $\tilde{O}(\sqrt{K})$ regret with respect to optimal, which together mean that the empirical regret is lower than other approaches.

\textbf{Limitations:}
Our goal is to find an RL algorithm with no constraint violation with high probability in the tabular setting. Since safe exploration is a basic problem, we follow the usual approach in the literature of first establishing the fundamental theory results for the tabular setting~\cite{jaksch2010near,dann2017unifying,azar2017minimax}. We also note that most of the existing work on exploration in safe RL is in the tabular setting~\cite{efroni2020exploration,singh2020learning,liu2021learning}. In our future work, we plan to employ the DOPE approach in a function approximation setting. 

\begin{figure*}[t] 

    \subfigure[{\footnotesize Objective Regret}]{\label{fig:objective-regretsimple}%
      \includegraphics[width=0.31\linewidth,height=2.8cm,clip=true,trim=4mm 4mm 3.5mm 3mm]{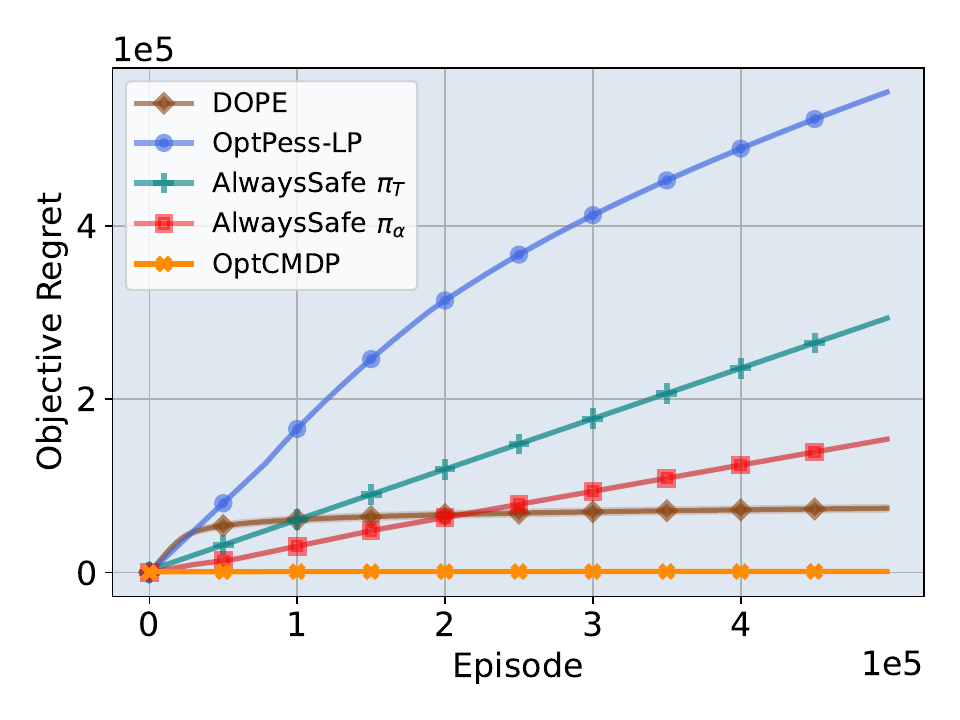}}%
    ~
    \subfigure[{\footnotesize Constraint Regret}]{\label{fig:constraint-regretsimple}%
      \includegraphics[width=0.31\linewidth,height=2.8cm,clip=true,trim=4mm 4mm 3.5mm 3mm]{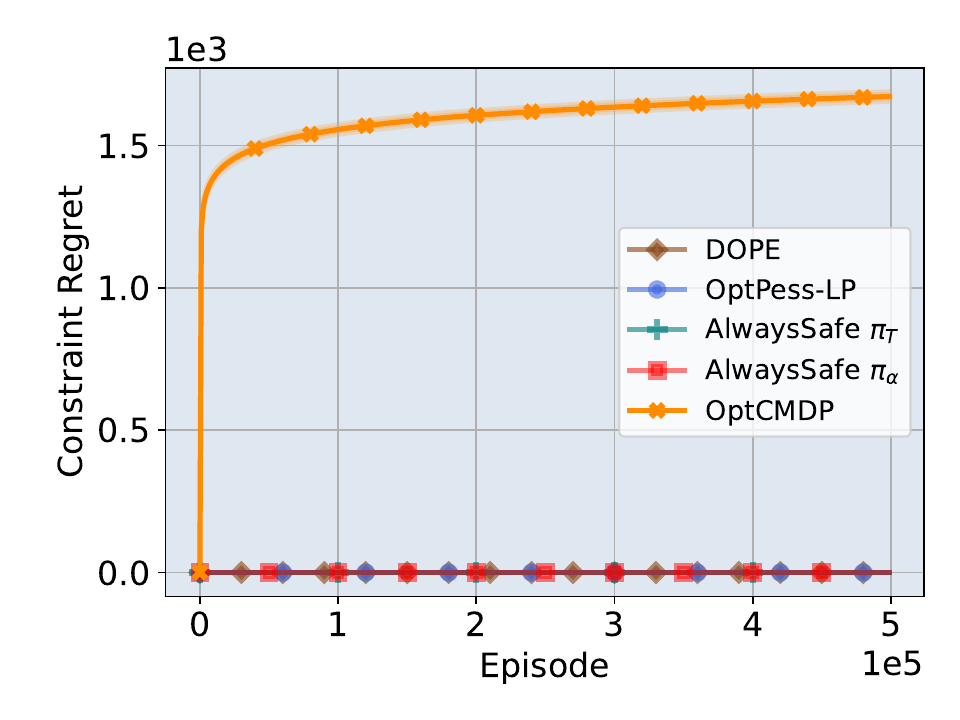}}%
    ~
      \subfigure[{\footnotesize DOPE Objective Regret, varying $\bar{C}_b$}]{\label{fig:objective-regretsimplecb}%
      \includegraphics[width=0.31\linewidth,height=2.8cm,clip=true,trim=4mm 4mm 3.5mm 3mm]
      {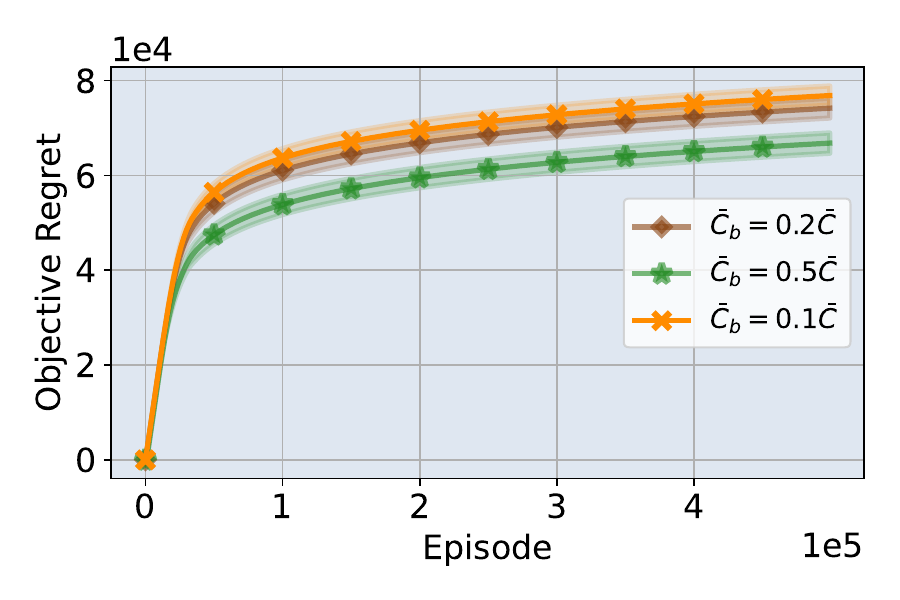}
      }
      \vspace{-0.1in}
  {\caption{Illustrating the Objective Regret and Constraint Regret for a Factored CMDP environment.}  \label{fig: simpleCMDPplots}}
\end{figure*}
\begin{figure*}[t]

    \subfigure[{\footnotesize Objective Regret}]{\label{fig:objective-regretmedia}%
      \includegraphics[width=0.31\linewidth,clip=true,trim=4mm 4mm 3.5mm 4mm]{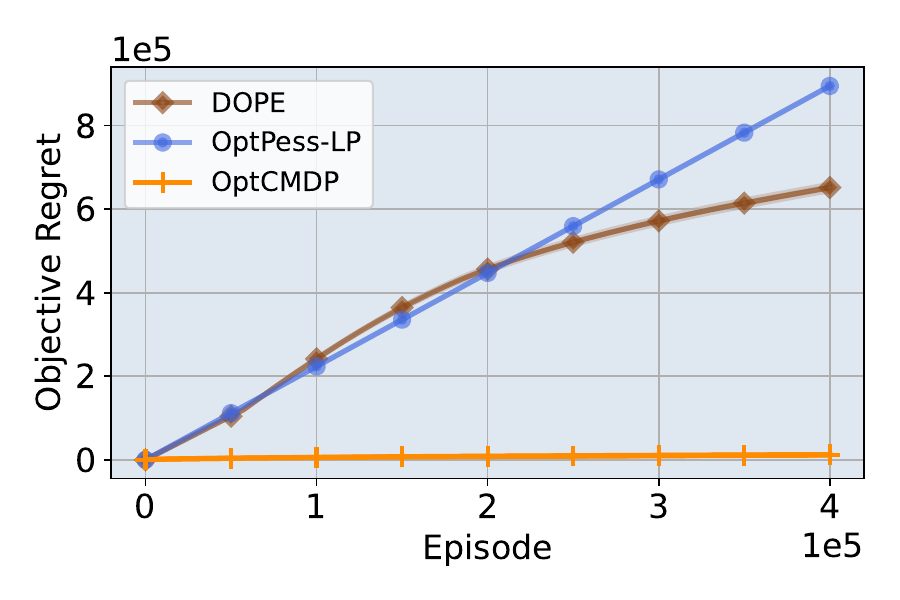}}%
    ~
    \subfigure[{\footnotesize Constraint Regret}]{\label{fig:constraint-regretMedia}%
      \includegraphics[width=0.31\linewidth,clip=true,trim=4mm 4mm 3.5mm 4mm]{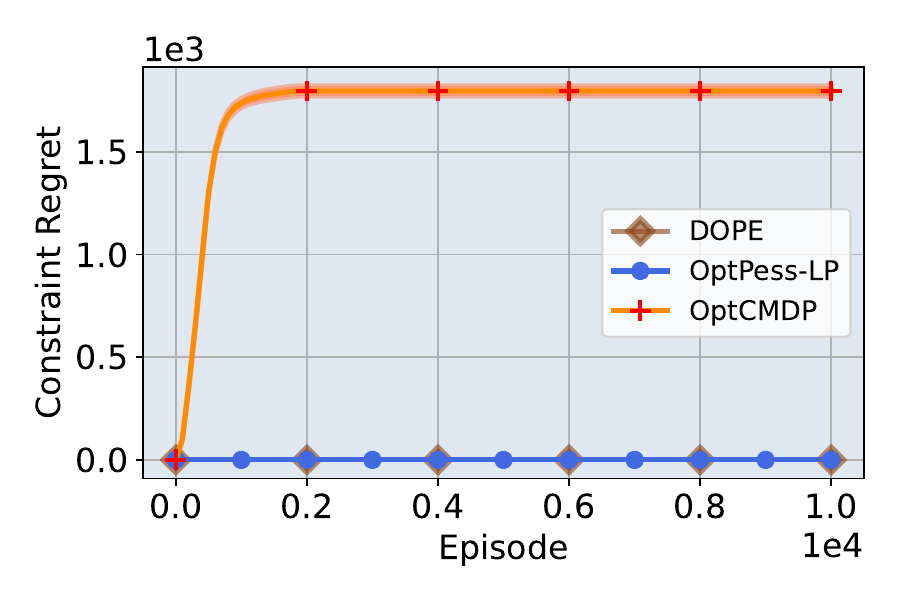}}
    ~
    \subfigure[{\footnotesize DOPE Objective Regret, varying $\bar{C}_b$}]{\label{fig:obj-regretMediacb}%
      \includegraphics[width=0.31\linewidth,clip=true,trim=4mm 4mm 3.5mm 4mm]{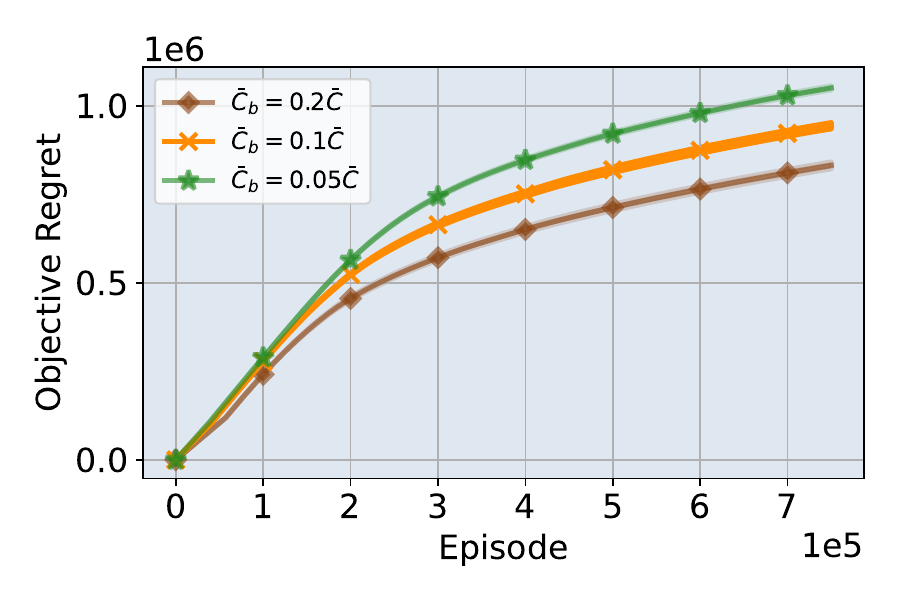}}      
  \vspace{-0.1in}
  {\caption{Illustrating the Objective Regret and Constraint Regret for the Media Streaming Environment.} \label{fig: MediaControlPlots}}

    \vspace{-0.1in}
\end{figure*}

\section{Conclusion}
We considered the safe exploration problem in reinforcement learning, wherein a safety constraint must be satisfied during learning and evaluation.   Earlier approaches to constrained RL have proposed optimism on the model, optimism on reward, and pessimism on constraints as means of modulating exploration, but none have shown order optimal regret, no safety violation, and good empirical performance simultaneously.  We started with the conjecture that double optimism combined with pessimism is the key to attaining the ideal balance for fast and safe exploration, and design DOPE that carefully combines these elements.  We showed that DOPE not only attains order-optimal $\tilde{{O}}(\sqrt{K})$ regret without violating safety constraints, but also reduces the best known regret bound by a factor of $\sqrt{|\S|}.$  Furthermore, it has significantly better empirical performance than existing approaches.  We thus make a case for adoption of the approach for real world use cases and extension to large scale RL problems using function approximation. 
\section{Acknowledgement}
This work was supported in part by the grants NSF-CAREER-EPCN 2045783, NSF ECCS 2038963, and ARO W911NF-19-1-0367. Any opinions, findings, and conclusions or recommendations expressed in this material are those of the authors and do not necessarily reflect the views of the sponsoring agencies.

\bibliographystyle{plain}
\bibliography{Neurips22CameraReady}
\clearpage

\section*{Societal Impact and Ethics Statement}
Reinforcement learning has much potential for application to a variety of cyber-physical systems, such as the power grid, robotics and other systems where guarantees on the operating region of the system must be met.  Our work provides a theoretical basis for the design of controllers that can be applied in such scenarios.  The approaches presented in the paper were tested on simulated environments, and did not involve any human interaction.  We do not see any ethical concerns with our research approach.

 A note of caution with our approach is that the policy generated is only as good as the training environment, and many examples exist wherein the policy generated is optimal according to its training, but violates basic truths known to human operators and could fail quite badly.  Indeed, our approach does not provide sample-path guarantees, and the system could well move into deleterious states for a small fraction of the time, which might be completely unacceptable and trigger hard fail safes, such as breakers in a power system.  Understanding the right application environments with excellent domain knowledge is hence needed before any practical success can be claimed.

\section*{Checklist}

\begin{enumerate}
\item For all authors...
\begin{enumerate}
  \item Do the main claims made in the abstract and introduction accurately reflect the paper's contributions and scope?
    \answerYes{In Theorem~\ref{thm:main-regret-theorem} and Section~\ref{sec:MainExperiments}}.
  \item Did you describe the limitations of your work?
    \answerYes{In Section~\ref{sec:MainExperiments}.}
  \item Did you discuss any potential negative societal impacts of your work?
    \answerYes{Societal Impact and Ethics Statement.}
  \item Have you read the ethics review guidelines and ensured that your paper conforms to them?
    \answerYes{Described in Societal Impact and Ethics Statement.}
\end{enumerate}

\item If you are including theoretical results...
\begin{enumerate}
  \item Did you state the full set of assumptions of all theoretical results?
    \answerYes{Assumption~\ref{assp:baseline}, and assumptions in statement of Theorem~\ref{thm:main-regret-theorem}}.
        \item Did you include complete proofs of all theoretical results?
    \answerYes{Appendix.}
\end{enumerate}

\item If you ran experiments...
\begin{enumerate}
  \item Did you include the code, data, and instructions needed to reproduce the main experimental results?
    \answerYes{Supplemental material}
  \item Did you specify all the training details (e.g., data splits, hyperparameters, how they were chosen)?
    \answerYes{Section~\ref{sec:MainExperiments}, and Appendix~\ref{sec:experiments}.}
        \item Did you report error bars (e.g., with respect to the random seed after running experiments multiple times)?
    \answerYes{Figures contain error bars.}
        \item Did you include the total amount of compute and the type of resources used (e.g., type of GPUs, internal cluster, or cloud provider)?
    \answerYes{Appendix~\ref{sec:experiments}}
\end{enumerate}

\item If you are using existing assets (e.g., code, data, models) or curating/releasing new assets...
\begin{enumerate}
  \item If your work uses existing assets, did you cite the creators?
    \answerNA{}
  \item Did you mention the license of the assets?
    \answerNA{}
  \item Did you include any new assets either in the supplemental material or as a URL?
    \answerYes{Code is released as supplemtental material.}
  \item Did you discuss whether and how consent was obtained from people whose data you're using/curating?
    \answerNA{}
  \item Did you discuss whether the data you are using/curating contains personally identifiable information or offensive content?
    \answerNA{}
\end{enumerate}

\item If you used crowdsourcing or conducted research with human subjects... \answerNA{}

\end{enumerate}

\newpage
\appendix
\section{Linear Programming  Method for Solving the CMDP Problem}
\label{sec:lp-cmdp}

Here we give a brief description on solving the CMDP problem \eqref{eq:CMDP-OPT} using the linear programming method when the model $P$ is known. The details can be found in \cite[Section 2]{efroni2020exploration}. 

The first step is to reformulate \eqref{eq:CMDP-OPT} using \textit{occupancy measure} \cite{altman1999constrained, zimin2013online}. For a given policy $\pi$ and an initial state $s_{1}$, the state-action occupation measure $w^{\pi}$ for the MDP with model $P$ is defined as
\begin{align}
    \label{eq:occupancy-measure-1}
    w^{\pi}_{h}(s,a;P) =  \mathbb{E}[\mathbbm{1}\{s_{h} = s, a_{h} = a\} | s_{1}, P, \pi] = \mathbb{P}(s_{h} = s, a_{h} = a | s_{1}, P, \pi). 
\end{align}
Given the occupancy measure, the policy that generated it can easily be   computed as
\begin{align}
\label{eq:occu-to-policy}
\pi_{h}(s, a) = \frac{w^{\pi}_{h}(s,a;P)}{\sum_{b} w^{\pi}_{h}(s,b;P)}.
\end{align}

The occupancy measure of any policy $\pi$ for an MDP with model $P$ should satisfy the following conditions. We omit the explicit dependence on $\pi$ and $P$ from the notation of $w$ for simplicity. 
\begin{align}
\label{eq:occupancy-conditions-1}
&\sum_{a} w_{h}(s,a) = \sum_{s^{\prime},a^{\prime}} P(s|s^{\prime},a^{\prime}) w_{h-1}(s^{\prime},a^{\prime}), \quad  \forall h \in [H] \setminus \{1\} \\
\label{eq:occupancy-conditions-2}
&\sum_{a} w_{1}(s,a) = \mathbbm{1}\{s = s_{1}\}, \quad \forall s \in \mathcal{S}, w_{h}(s,a) \geq 0, \quad \forall (s,a, h) \in  \mathcal{S} \times \mathcal{A} \times [H] 
\end{align}
From the above conditions, it is easy to show that $\sum_{s,a} w_{h}(s,a) = 1$. So, occupancy measures are indeed probability measures. Since the set of occupancy measures for a model $P$, denoted as $\mathcal{W}(P)$, is defined by a set of affine constraints, it is straight forward to show that $\mathcal{W}(P)$ is convex. We state this fact formally below. 

\begin{proposition}
\label{thm:occupancy-convex-set}
The set of occupancy measures for an MDP with model $P$, denoted as $\mathcal{W}(P)$, is convex. 
\end{proposition}
Recall that the   value of a policy $\pi$ for an arbitrary cost function  $l: \S \times \A \rightarrow \mathbb{R}$ with a given initial state $s_{1}$ is defined as $V_l^{\pi}(P) = \E [ \sum_{h=1}^H l(s_h,a_h) | s_{1} =s, \pi, P ]$. It can then be expressed using the occupancy measure as
\begin{align*}
V_l^{\pi}(P) =  \sum_{h,s,a}  w_{h}^{\pi}(s,a;P) ~l_h(s,a) = l^{\top} w^{\pi}(P),
\end{align*}
where $w^{\pi}(P) \in \mathbb{R}^{SAH}$ with $(s,a,h)$ element is given by $w_{h}^{\pi}(s,a;P)$ and $l \in \mathbb{R}^{SAH}$ with $(s,a,h)$ element is given by $l_h(s,a)$. The CMDP problem  \eqref{eq:CMDP-OPT} can then be written as
\begin{align}
\label{eq:occu-cmdp-1}
\min_{\pi} ~r^{\top} w^{\pi}(P) ~~~ \text{s.t.} ~~~   c^{\top} w^{\pi}(P) \leq \bar{C}. 
\end{align}

Using the properties of the occupancy measures, the reformulated CMDP problem \eqref{eq:occu-cmdp-1} can be rewritten as an LP, where the optimization variables are occupancy measures \cite{zimin2013online, efroni2020exploration} .  More precisely, the CMDP problem \eqref{eq:CMDP-OPT} and its equivalent \eqref{eq:occu-cmdp-1} can be written as
\begin{subequations}
\label{eq:cmdp-LP}
\begin{align}
\min_{w} ~~~ &\sum_{h,s,a} w_{h}(s,a) r_h(s,a) \\
\text{subject to}~~~&\sum_{h,s,a} w_{h}(s,a)c_h(s,a) \leq \bar{C} \\
&\sum_{a} w_{h}(s,a) = \sum_{s^{\prime},a^{\prime}} P_{h-1}(s|s^{\prime},a^{\prime}) w_{h-1}(s^{\prime},a^{\prime}),  \forall h \in [H] \setminus \{1\} \\
&\sum_{a} w_{1}(s,a) = \mathbbm{1}\{s = s_{1}\}, \quad \forall s \in \mathcal{S} \\
&w_{h}(s,a) \geq 0, \quad \forall (s,a, h) \in  \mathcal{S} \times \mathcal{A} \times [H] 
\end{align}
\end{subequations}

From the optimal solution $w^{*}$ of \eqref{eq:cmdp-LP}, the optimal policy $\pi^{*}$ for the CMDP problem \eqref{eq:CMDP-OPT} can be computed using \eqref{eq:occu-to-policy}.

 
 \section{Extended Linear Programming Method for Solving OFU and DOP problems}
 \label{sec:extended-lp-cmdp}

The OFU problem \eqref{eq:ofu-optimization} and the DOP problem  \eqref{eq:DOPE-optimization} may appear much more challenging than the CMDP problem \eqref{eq:CMDP-OPT} because they involve a minimization over all models in $\mathcal{P}_{k}$, which is non-trivial. However, finding the optimistic model (and the corresponding optimistic policy) from a given confidence set  is  a standard step in OFU style algorithms for exploration in RL \cite{jaksch2010near,efroni2020exploration}. In the case of standard (unconstrained) MDP, this problem is solved using a approach called \textit{extended value iteration} \cite{jaksch2010near}. In the case of constrained MDP,   \eqref{eq:ofu-optimization} (and similarly   \eqref{eq:DOPE-optimization} ) can be solved by an approach called \textit{extended linear programming}. The details are given in \cite{efroni2020exploration}. We give a brief description below for completeness. Note that the description below mainly focus on solving  \eqref{eq:ofu-optimization}. Solving  \eqref{eq:DOPE-optimization} is identical, just by replacing the constraint cost function $c_{h}(\cdot, \cdot)$ with pessimist constraint cost function $\bar{c}_{h,k}(\cdot, \cdot)$, $\forall h \in [H]$, and  is mentioned at the end of this subsection. 

Define the state-action-state occupancy measure $z^{\pi}$ as $z_{h}^{\pi}(s,a,s';P) = P_h(s'|s,a) w_{h}^{\pi}(s,a;P)$. The extended LP formulation corresponding to \eqref{eq:ofu-optimization} is then given as follows: 
\begin{subequations}
\label{eq:cmdp-ExtendedLP}
\begin{align}
    \max_{z} ~~~ & \sum_{s, a, s', h} z_{h}(s, a, s') r_h(s, a)\\
   \text{s.t.}~~~&\sum_{s, a, s', h} z_{h}(s, a, s') {c}_h(s, a) \leq \bar{C}\\
       \label{eq:ExtendedLP-ec-3}
    & \sum_{a, s'} z_{h}(s, a, s') = \sum_{s', a'}  z_{h-1}(s', a', s) ~~ \forall h \in [H] \setminus \{1\}, s \in \S \\
    & \sum_{a, s'} z_{1}(s, a, s') =  \mathbbm{1}\{s = s_{1}\}, \quad \forall s \in \mathcal{S} \\
    &z_{h}(s, a, s') \geq 0, ~~ \forall (s, a, s', h) \in \S \times \A \times \S \times [H],\\
    & z_{h}(s, a, s') - (\widehat{P}_{h,k}(s'|s, a) + \beta^p_{h,k}(s,a,s')) \sum_{y} z_{h}(s, a, y) \leq 0,  \nonumber \\
    \label{eq:ExtendedLP-ec-1}
    &\hspace{6cm} ~~\forall (s, a, s', h) \in \S \times \A \times \S \times [H] \\
    & -z_{h}(s, a, s') + (\widehat{P}_{h,k}(s'|s, a) - \beta^p_{h,k}(s,a,s'))  \sum_{y} z_{h}(s, a, y) \leq 0, \nonumber \\
    \label{eq:ExtendedLP-ec-2}
   &\hspace{6cm} ~~\forall (s, a, s', h) \in \S \times \A \times \S \times [H] 
 \end{align}
 \end{subequations}

The last two conditions (\eqref{eq:ExtendedLP-ec-1} and \eqref{eq:ExtendedLP-ec-2})  distinguish the extended LP formulation from the LP formulation for CMDP. These constraints are based on the Bernstein confidence sets around the empirical model $\widehat{P}_{k}$.

From the solutions $\tilde{z}^{*}$ of the extended LP, we can obtain the solution of \eqref{eq:ofu-optimization} as
\begin{align}
\label{eq:z-to-pi-P}
\underline{P}_{h,k}(s^{\prime}|s,a) = \frac{\tilde{z}_h^{*}(s,a,s^{\prime})}{\sum_y \tilde{z}_h^{*}(s,a,y)}, \quad
\underline{\pi}_{h,k}(s, a) = \frac{\sum_{s^{\prime}} \tilde{z}_h^{*}(s,a,s^{\prime})}{\sum_{b,s^{\prime}} \tilde{z}_h^{*}(s,b,s^{\prime})}.
\end{align}

\section{Useful Technical Results}
\label{sec:useful-lemmas}

Here we reproduce the   supporting technical results that are  required  for analyzing our DOPE algorithm. We begin by stating the following concentration inequality,  known as empirical Bernstein inequality \cite[Theorem 4]{MaurerP09}.

\begin{lemma}[Empirical Bernstein Inequality] 
\label{lem:Bernstein}
Let $Z=(Z_1,\dots Z_n)$ be i.i.d random vector with values in~$[0,1]^n$, and let $\delta \in (0, 1)$. Then, with probability at least $1-\delta$, it holds that
\begin{align*}
\mathbb{E}[Z] - \frac{1}{n} \sum_{i=1}^n Z_i \leq \sqrt{\frac{2V_n(Z) \log(\frac{2}{\delta})}{n}}+\frac{7 \log(\frac{2}{\delta})}{3(n-1)},
\end{align*}
where $V_n(Z)$ is the sample variance.  
\end{lemma}

We can get the following result using  empirical Bernstein inequality and union bound. This result is widely used in the literature now, for example see \cite[Proof of Lemma 2]{jin20c}, 
\begin{lemma} 
\label{lem:bernstein-bound-for-p-1}
With probability at least $1-2\delta$, for all $(h,s, a , s') \in [H] \times \S \times \A \times \S$, $k \in [K]$, we have 
\begin{align*}
|P_h(s'|s, a) - \widehat{P}_{h,k} (s'|s, a) | \leq \sqrt{\frac{4 \textnormal{Var}(\widehat{P}_{h,k} (s'|s, a))  \log \left( \frac{2 S A K H}{\delta} \right)  }{n_{h,k}(s,a) \vee 1} } +  \frac{14  \log \left( \frac{2 S A K H}{\delta} \right)}{3(({n_{h,k}(s,a) - 1)  \vee 1}))}.
\end{align*}
\end{lemma}

Recall (from \eqref{eq:confidenceset} - \eqref{eq:beta}) that
\begin{align*}
&\beta^p_{h,k}(s, a, s') = \sqrt{\frac{4 \textnormal{Var}(\widehat{P}_{h,k} (s'|s, a))  \log \left( \frac{2 S A K H}{\delta} \right)  }{n_{h,k}(s,a) \vee 1} } +  \frac{14  \log \left( \frac{2 S A K H}{\delta} \right)}{3({n_{h,k}(s,a)   \vee 1}))}, \\
&\mathcal{P}_{h,k}(s, a) =  \{{P'} :| {P'_h}(s'|s,a) -\widehat{P}_{h,k}(s'|s,a)| \leq \beta^p_{h,k}(s,a,s'), \forall s'\in\S\},
\end{align*}
and define $
\mathcal{P}_{k} = \bigcap_{(h,s,a) \in [H] \times \S \times \A} 
\mathcal{P}_{h,k}(s, a)$.

Define the event  
\begin{align}
\label{eq:P-event}
F^{p} = \left\lbrace  P \in \mathcal{P}_{k}, ~ \forall k \in [K]  \right\rbrace.
\end{align}
Then, using Lemma \ref{lem:bernstein-bound-for-p-1}, we can get the following result immediately. 
\begin{lemma} 
\label{lem:Fp-event}
Let $F^{p}$ be the event defined as in \eqref{eq:P-event}. Then, 
$\mathbb{P}(F^{p}) \geq 1 - 2 \delta$.
\end{lemma}

Define the events $F^c_k = \{\forall (h,s,a) : |\hat{c}_{h,k}(s,a) - c_h(s,a)| \leq  \beta_{h,k}^l(s,a)\}$, and $F^r_k = \{\forall h,s,a : |\hat{r}_{h,k}(s,a) - r_h(s,a)| \leq  \beta_{h,k}^l(s,a)\}$, and define 

\begin{equation}
\label{eq:r-event}
F^l = \bigcap_{k} F_k^c \cap F_k^r 
\end{equation}

The following is a standard result, and can be obtained by Hoeffding's inequality, and using a union bound argument on all $h,s,a$ and all possible values of $n_{h,k}(s,a)$, for all $k \in [K]$.
\begin{lemma}\label{lem:r-goodevent}
$\mathbb{P}(F^l) \ge 1-2\delta$.
\end{lemma}

We now define the  event $F_{w}$ as follows
\begin{align}
\label{eq:rho-event}
&F^{w} \nonumber \\
&= \left\lbrace  n_{h,k}(s, a) \geq \frac{1}{2} \sum_{j < k} w_{h,j}(s, a) - H\log \frac{S A H}{\delta},~  \forall  (h,s, a , s', k) \in [H] \times \S \times \A \times \S \times [K]  \right\rbrace,
\end{align}
where $w_{h,j}$ is the occupancy measure corresponding to the policy chosen in episode $j$.
We have the following result from \cite[Corollary E.4.]{dann2017unifying}
\begin{lemma}[Corollary E.4., \cite{dann2017unifying}]
\label{lem:Frho-event}
Let $F^{w}$ be the event defined as in \eqref{eq:rho-event}. Then, 
$\mathbb{P}(F^{w}) \geq 1 -  \delta$.
\end{lemma}

We now define the \textbf{good event} $G = F^{p} \cap F^{w} \cap F^l$.  Using union bound, we can show that $\mathbb{P}(G) \geq 1 - 5 \delta$. Since our analysis is based on this good event, we formally state it as a lemma.
\begin{lemma}
\label{lem:Good-event}
Let $F^{p}$ is  defined as in \eqref{eq:P-event} and $F^{l}$ defined in~\eqref{eq:r-event}, $F^{w}$ is defined as in \eqref{eq:rho-event}. Let the good event $G = F^{p} \cap F^{w} \cap F^l$. Then, $\mathbb{P}(G) \geq 1 - 5 \delta$. 
\end{lemma}

We will also use the following results for analyzing the performance of our DOPE algorithm.  
\begin{lemma}[Lemma 36, \cite{efroni2020exploration}]
\label{lem:nk-sum-lemma-1}
Under the event $F^{w}$,
\begin{align*}
\sum_{k=1}^K\sum_{h=1}^H \mathbb{E} \left[ \frac{1}{\sqrt{n_{h,k}(s_{h,k},a_{h,k})\vee 1}}|\mathcal{F}_{k-1}\right] 
\leq \tilde{\mathcal{O}}(\sqrt{S A H^2 K} + S A H).
\end{align*}
\end{lemma}

\begin{lemma}[Lemma 37, \cite{efroni2020exploration}]
\label{lem:nk-sum-lemma-2}
Under the event $F^{w}$,
\begin{align*}
\sum_{k=1}^K \sum_{h=1}^H \mathbb{E} \left[\frac{1}{n_{h,k}(s_{h,k},a_{h,k}) \vee 1}|\mathcal{F}_{k-1} \right] \leq  \tilde{\mathcal{O}}(S A H^2).
\end{align*}
\end{lemma}

\begin{lemma}[Lemma 8,\cite{jin20c}]
\label{lem: confidence}
Under the event $G$, for all $k,h,s,a,s'$, and for all $P' \in \mathcal{P}_k$, there exists constants $C_1, C_2>0$ such that $|P'_h(s'|s,a) - P_h(s'|s,a)| \leq C_1\sqrt{\frac{P_h(s'|s,a)L}{n_{h,k}(s,a) \vee 1}} + C_2\frac{L}{n_{h,k}(s,a) \vee 1}$.
\end{lemma}

\begin{lemma}[Value difference lemma] 
\label{lem:value-difference}
Consider two MDPs $M = (\mathcal{S},\mathcal{A},l,P)$ and $M' = (\mathcal{S},\mathcal{A},l',P')$. For any policy $\pi$, state $s \in \S$, and time step $h \in [H]$, the following relation holds.
\begin{align*}
&V^{\pi}_{l,h}(s;P) -V^{\pi}_{l',h}(s;P') \\ &=\mathbb{E}\left[\sum_{\tau=h}^{H} (l_{\tau}(s_{\tau},a_{\tau}) -l'_{\tau}(s_{\tau},a_{\tau})) + ((P_{\tau}-P'_{\tau})(\cdot|s_{\tau},a_{\tau}))^{\top} V^{ \pi}_{l,\tau+1}(\cdot;P)|s_{h}=s,\pi,P'\right]   \\
&=\mathbb{E}\left[\sum_{\tau=h}^{H} (l'_{\tau}(s_{\tau},a_{\tau}) -l_{\tau}(s_{\tau},a_{\tau})) + ((P'_{\tau}-P_{\tau})(\cdot|s_{\tau},a_{\tau}))^{\top} V^{ \pi}_{l,\tau+1}(\cdot;P')|s_{h}=s,\pi,P\right]  .
\end{align*}
\end{lemma}

\section{Proof of the Main Results} 
All the results we prove in this section are conditioned on the  good event $G$ defined in Section \ref{sec:useful-lemmas}. So, the results hold true with a  probability greater than $1  - 5 \delta$ according to Lemma \ref{lem:Good-event}. We will omit stating this conditioning under $G$ in each statement to avoid repetition.  

\subsection{Proofs of Proposition \ref{prop:Feasibility}} \label{subsec:prop1}

First note that
\begin{align}
\mathbb{E}[\sum^{H}_{h=1}  \bar{c}_{h,k}(s_{h}, a_{h}) | \pi', P', \mathcal{F}_{k-1}]  &=\mathbb{E}[\sum^{H}_{h=1}  \left(\hat{c}_{h,k}(s_{h}, a_{h})+ \beta^l_{h,k}(s_h, a_{h}) + H \bar{\beta}^p_{h,k}(s_{h}, a_{h}) \right)| \pi', P'] \nonumber \\
&=V^{\pi'}_{\hat{c}_k}(P') + \eta^{\pi'}_{k}(P')+ \epsilon^{\pi'}_{k}(P'), \label{eq:Vk-decomp-1}
\end{align} 
\begin{align}
\mathbb{E}[\sum^{H}_{h=1}  \bar{r}_{h,k}(s_{h}, a_{h}) | \pi', P', \mathcal{F}_{k-1}]&=\mathbb{E}[\sum^{H}_{h=1}  \left(\hat{r}_{h,k}(s_{h}, a_{h}) + \frac{3H}{\bar{C}-\bar{C}_b} \beta^l_{h,k}(s^{h}, a^{h}) \right. \nonumber\\
& \hspace{40mm} \left. - \frac{H^2}{\bar{C}-\bar{C}_b}\bar{\beta}^p_{h,k}(s_{h}, a_{h}) \right) | \pi', P']  \nonumber \\
\label{eq:Vk-decomp-2} 
& =V^{\pi'}_{\hat{r}_k}(P') - \frac{3H}{\bar{C}-\bar{C}_b} \eta^{\pi'}_{k}(P') - \frac{H}{\bar{C}-\bar{C}_b}  \epsilon^{\pi'}_{k}(P'),
\end{align}
where equations~\eqref{eq:Vk-decomp-1},~\eqref{eq:Vk-decomp-2} are due to linearity of expectation.
\begin{lemma}
\label{lem:epsilon-k-sum-bound}
Let $\epsilon^{\pi'}_{k}(P')$ and $\eta^{\pi'}_{k}(P')$ be as defined in \eqref{eq:epsilon-k-pi-P}. Also, let $\{\pi_{k}\}$ be the sequence of policies generated by  DOPE algorithm. Then,   for any $K' \leq K$, each of the following relations hold with a probability greater than $1 - 5 \delta$.
\[
\sum^{K'}_{k=1} \epsilon^{\pi_{k}}_{k}(P) \leq  \tilde{\mathcal{O}}(S \sqrt{A H^{4} K'  }), \hspace{0.1cm} \text{and}, \hspace{0.1cm}
\sum^{K'}_{k=1} \eta^{\pi_{k}}_{k}(P) \leq  \tilde{\mathcal{O}}(S \sqrt{A H^{2} K'  }).
\]
\end{lemma}
\vspace{-1cm}
\begin{proof}
\begin{align}
&\sum^{K'}_{k=1} \epsilon^{\pi_{k}}_{k}(P) 
= H \sum^{K'}_{k=1}  \sum^{H}_{h=1} \mathbb{E}[ \sum_{s'}  {\beta}^p_{h,k}(s_{h,k}, a_{h,k}, s') | \pi_{k}, P, \mathcal{F}_{k-1}]   \nonumber \\
& \stackrel{(a)}{\leq} H  \sum^{K'}_{k = 1} \mathbb{E} \left[ \sum_{h=1}^H \sqrt{\frac{4 L}{n_{h,k}(s_{h,k}, a_{h,k})\vee 1}} \sum_{s' \in \S} \sqrt{\widehat{P}_{h,k}(s'|s_{h,k}, a_{h,k})} | \pi_{k}, P, \mathcal{F}_{k-1}\right] \nonumber\\
&\hspace{3cm}+H S \sum^{K'}_{k = 1} \mathbb{E} \left[ \sum_{h=1}^H \frac{(14/3) L}{n_{h,k}(s_{h,k}, a_{h,k})\vee 1} | \pi_{k}, P, \mathcal{F}_{k-1}\right] \nonumber \\
&\stackrel{(b)}{\leq} 2 H \sqrt{S} \sqrt{L}  \sum^{K'}_{k = 1} \mathbb{E} \left[ \sum_{h=1}^H \sqrt{\frac{1}{n_{h,k}(s_{h,k}, a_{h,k})\vee 1}} | \pi_{k}, P, \mathcal{F}_{k-1}\right] \nonumber\\
&\hspace{3cm}+ (14/3) H S L \sum^{K'}_{k = 1} \mathbb{E} \left[ \sum_{h=1}^H \frac{1}{n_{h,k}(s_{h,k}, a_{h,k})\vee 1} | \pi_{k}, P, \mathcal{F}_{k-1}\right] \nonumber \\
\label{eq:feasibility-proof-eq1}
&\stackrel{(c)}{\leq}  H \sqrt{S}  \sqrt{L}   \tilde{\mathcal{O}}(\sqrt{S A H^2 K'} + S A H) + HS  L \tilde{\mathcal{O}}(S A H) \leq \tilde{\mathcal{O}}(S \sqrt{A H^{4} K'  }).
\end{align} 
Here, we get inequality $(a)$ by the definition of $\beta^p_{h,k}$ (c.f. \eqref{eq:beta}). To get $(b)$, note that $\sum_{s' \in \S} \sqrt{\widehat{P}_{h,k}(s'|(s_{h,k},a_{h,k}))}  \leq \sqrt{\sum_{s'} \widehat{P}_{h,k}(s'|(s_{h,k},a_{h,k}))} \sqrt{S}$ by Cauchy-Schwarz inequality and  $\sum_{s'} \widehat{P}_{h,k}(s'|(s_{h,k},a_{h,k})) = 1$. We get $(c)$ using Lemma \ref{lem:nk-sum-lemma-1} and Lemma \ref{lem:nk-sum-lemma-2}. 

The other part can also be obtained similarly from Lemma~\ref{lem:nk-sum-lemma-1}.
\end{proof}
We now give the proof of Proposition \ref{prop:Feasibility}.
\begin{proof}[\textbf{Proof of Proposition \ref{prop:Feasibility}}]
First note that even though $(\pi_{b}, P)$ is a feasible solution for the original CMDP problem \eqref{eq:CMDP-OPT}, it may not feasible for the DOP problem \eqref{eq:DOPE-optimization}. To see this, note that since   $
V^{\pi_{b}}_{\bar{c}_{k}}(P) =V^{\pi_{b}}_{\hat{c}_k}(P) + \eta^{\pi_{b}}_{k}(P) + \epsilon^{\pi_{b}}_{k}(P)$ and $V_{\hat{c}_k}^{\pi_b}(P) \leq V_c^{\pi_b}(P) + \eta_k^{\pi_b}(P) $, and $V^{\pi_{b}}_{{c}}(P) =  \bar{C}_{b}$,  we will have  $V^{\pi_{b}}_{\bar{c}_{k}}(P) \leq \bar{C}$ if $ 2 \eta^{\pi_b}_k(P) + \epsilon^{\pi_{b}}_{k}(P) \leq (C - \bar{C}_{b})$.   So, $(\pi_{b}, P)$ is  a feasible solution for  \eqref{eq:DOPE-optimization} if,  $2 \eta^{\pi_{b}}_{k}(P) + \epsilon^{\pi_{b}}_{k}(P) \leq (C - \bar{C}_{b})$. This is a sufficient condition for the feasibility of $(\pi_b,P)$. This condition may not be satisfied in the  initial episodes.

However, since  $\eta^{\pi_b}_k(P) $ and $\epsilon^{\pi_{b}}_{k}(P)$ are decreasing in $k$, if  $(\pi_{b}, P)$  becomes a feasible solution for  \eqref{eq:DOPE-optimization} at episode $k'$, then it will remain to be a feasible solution for all episodes $k \geq k'$.

Suppose $\pi_{k} = \pi_{b}$ for all $k \leq K'$. Also, suppose  the above condition  is not satisfied in the algorithm until episode $K'+1$. Then,  $2 \eta^{\pi_b}_k(P) + \epsilon^{\pi_{b}}_{k}(P) > C - \bar{C}_{b}$ for all $k \leq  K'$. So, we should get 
\begin{align*}
K' (\bar{C} - \bar{C}_{b}) < \sum^{K'}_{k=1}  2 \eta^{\pi_{b}}_{k}(P) + \epsilon^{\pi_{b}}_{k}(P) = \sum^{K'}_{k=1}   2\eta^{\pi_{k}}_{k}(P) + \epsilon^{\pi_{k}}_{k}(P)   {\leq}  \tilde{\mathcal{O}}(S \sqrt{A H^{4} K' }),
\end{align*}
where the last inequality is from Lemma  \ref{lem:epsilon-k-sum-bound}. However, this inequality is violated for $K' \geq  \tilde{\mathcal{O}}( \frac{S^{2} A H^{4}}{ (\bar{C} - \bar{C}_{b})^{2}})$. So, $(\pi_{b}, P)$  is  a feasible solution for  \eqref{eq:DOPE-optimization}  for any episode $k \geq K_{o} =   \tilde{\mathcal{O}}( \frac{S^{2} A H^{4}}{ (\bar{C} - \bar{C}_{b})^{2}})$ provided that $\pi_{k} = \pi_{b}$ for all $k \leq K_{o}$. 
\end{proof}
The above result, however, only shows that $\pi_{b}$ becomes a feasible policy after some finite number of episodes. A natural question is, is $\pi_{b}$ the only feasible policy? 
In such a case, the DOPE algorithm may not provide enough  exploration to learn the optimal policy. 

We alleviate the concerns about the above possible issue by showing that for all $k \geq K_{o}$, there exists a feasible solution $(\pi', P)$ for the OP problem  \eqref{eq:DOPE-optimization} such that $w_{h}^{\pi'}(s, a;P) > 0$ for every $(s, a) \in \S \times \A$ with $w_{h}^{\pi^{*}}(s, a;P) > 0$. Informally, this implies that  $\pi'$  will visit all state-action pairs that will be visited by the optimal policy $\pi^{*}$. This result can be derived as a corollary for Proposition \ref{prop:Feasibility}. 

\subsection{Proof of Proposition \ref{prop:safety} }

\begin{proof}
For any episode  $k \leq K_{o}$, we have $\pi_{k} = \pi_{b}$, and it is safe by Assumption \ref{assp:baseline}. For $k \geq K_{o}$, \eqref{eq:DOPE-optimization} is feasible according to Proposition \ref{prop:Feasibility}.  Since  $({\pi}_{k}, {P}_k)$ is the solution of \eqref{eq:DOPE-optimization}, we have  $V_{\bar{c}_k}^{{\pi}_{k}}({P}_k)$ $\leq \bar{C}$. We will now show that $V_{{c}}^{{\pi}_{k}}(P)$ $\leq \bar{C}$, conditioned on the good event $G$. 

By the  value difference lemma (Lemma \ref{lem:value-difference}), we have
\begin{align}
V^{ {\pi}_{k}}_{c}(P) -V^{{\pi}_{k}}_{c}({P}_{k}) &=\mathbb{E} [\sum_{h=1}^{H} ((P_h-{P}_{h,k})(\cdot|s_{h,k},a_{h,k}))^{\top} V^{  {\pi}_{k}}_{c,h+1}(\cdot;P)| {\pi}_{k},{P}_{k}, \mathcal{F}_{k-1}]  \nonumber \\
&\stackrel{(a)}{\leq} \mathbb{E} [\sum_{h=1}^{H} \| ((P_h-{P}_{h,k})(\cdot|s_{h,k},a_{h,k}))\|_{1} \| V^{ {\pi}_{k}}_{c,h+1}(\cdot;P) \|_{\infty}| {\pi}_{k},{P}_{k}, \mathcal{F}_{k-1}]  \nonumber \\
\label{eq:feasibility-pik-1}
&\stackrel{(b)}{\leq}  H \mathbb{E} [\sum_{h=1}^{H}  \bar{\beta}^p_{h,k}(s_{h,k},a_{h,k}) | {\pi}_{k}, {P}_{k}, \mathcal{F}_{k-1}]  = \epsilon^{{\pi}_{k}}_{k}({P}_{k}).
\end{align}
Here, we get $(a)$ by Holder's inequality inequality. To get $(b)$, we make use of two observations. First, note    that $\| V^{  {\pi}_{k}}_{c,h+1}(\cdot;P) \|_{\infty} \leq H$  because the expected cumulative cost cannot be grater than $H$ since  $|c(\cdot, \cdot)| \leq 1$ by assumption. Second, under the good event $G$, $\sum_{s'} |P_h(s'|s,a) -{P}_{h,k}(s'|s,a)| \leq  \sum_{s'} \beta^p_{h,k}(s,a,s') = \bar{\beta}^p_{h,k}(s,a)$.

From \eqref{eq:feasibility-pik-1},  we get
\begin{align*}
V^{{\pi}_{k}}_{c}(P) \leq  V^{{\pi}_{k}}_{c}({P}_{k}) + \epsilon^{{\pi}_{k}}_{k}({P}_{k})  &\stackrel{(c)}{\leq} V_{\hat{c}_k}^{\pi_k}(P_k) + \eta_k^{\pi_k}(P_k) + \epsilon^{{\pi}_{k}}_{k}({P}_{k}) = V^{{\pi}_{k}}_{\bar{c}_{k}}({P}_{k})  \stackrel{(d)}{\leq} \bar{C},
\end{align*}
where $(c)$ is by definition of good event and $(d)$ is from the fact that $({\pi}_{k}, {P}_k)$ is the solution of \eqref{eq:DOPE-optimization}.  So, $V^{{\pi}_{k}}_{c}(P) \leq \bar{C}$, and hence $\pi_{k}$ is safe, under the good event $G$. So, this statement holds with a probability greater than $1 - 5 \delta$, according to Lemma \ref{lem:Good-event}.  
\end{proof}

\subsection{Proof of Theorem~\ref{thm:main-regret-theorem}} \label{subsec:theorem1}
We first prove an important lemma. 
\begin{lemma}[Optimism]
\label{lem:termII-decomp-Optimism}
Let $(\pi_{k}, P_{k})$ be the optimal solution corresponding to the DOP problem \eqref{eq:DOPE-optimization}. Then,
\begin{align*}
 V^{\pi_{k}}_{\bar{r}_k}(P_{k}) \leq V^{\pi^*}_{r}(P).
\end{align*}
\end{lemma}
\vspace{-0.65cm}
\begin{proof}
We will first consider a more general version of the DOP problem \eqref{eq:DOPE-optimization} as 
\begin{align}
\label{eq:modified-DOPE-optimization1}
(\tilde{\pi}_{k}, \tilde{P}_{k}) = \argmin_{{\pi'}, {P'} \in \mathcal{P}_{k}}~~  V_{\tilde{r}_k}^{\pi'}(P')  ~~ \text{subject to} ~~   V_{\bar{c}_{k}}^{\pi'}(P') \leq \bar{C},
\end{align}
where we change $\bar{r}_{k}$ in \eqref{eq:DOPE-optimization}  to $\tilde{r}_{k}$ above, with $\tilde{r}_h^{k}(s, a) = \hat{r}_h(s, a) - 3b \beta_{h,k}^l(s,a) -  b H \bar{\beta}_{h,k}^p(s, a)$, for $b>0$. Note that  \eqref{eq:modified-DOPE-optimization1} reduces to  \eqref{eq:DOPE-optimization} for $b = \frac{H}{\bar{C}-\bar{C}_b}$ and hence it is indeed a general version.

Using the occupancy measures $w_{h}^{\pi_{b}}$ and  $w_{h}^{\pi^{*}}$, define a new occupancy measure $
\tilde{w}_{h}(s, a) = ( 1- \alpha_{k}) w_{h}^{\pi_{b}} (s, a;P) + \alpha_{k} w_{h}^{\pi^{*}}(s, a;P)  
$
for an $\alpha_{k} > 0$.

Note that $\tilde{w}$ is a valid occupancy measure since the set of occupancy measure is convex (c.f.  Proposition \ref{thm:occupancy-convex-set}).  Let $\tilde{\pi}$ be the policy corresponding to the occupancy measure $\tilde{w}$, which can be obtained according to \eqref{eq:occu-to-policy} so that $\tilde{w} = w^{\tilde{\pi}}$.

\textbf{Claim 1:} $(\tilde{\pi}, P)$ is a feasible solution for \eqref{eq:modified-DOPE-optimization1} when $\alpha_{k}$ satisfies the sufficient condition
\begin{align}
\label{eq:alpha-condition1}
\alpha_{k} \leq \frac{\bar{C} - \bar{C}_{b} - (\epsilon^{\pi_b}_{k}(P) + 2 \eta_k^{\pi_b}(P))}{\bar{C} - \bar{C}_{b}  +  (\epsilon^{\pi^{*}}_{k}(P) + 2 \eta_k^{\pi^*}(P)) -  (\epsilon^{\pi_b}_{k}(P)+ 2 \eta_k^{\pi_{b}}(P))} .
\end{align}

\textit{Proof of Claim 1:} Since value function is a linear  function of the  occupancy measure, we have
\begin{align*}
   &V_{\bar{c}_{k}}^{\tilde{\pi}}(P) = (1-\alpha_k) V_{\bar{c}_{k}}^{\pi_b}(P) +  \alpha_k V_{\bar{c}_{k}}^{\pi^*}(P)\\
   &= (1-\alpha_k) (V_{\hat{c}_{k}}^{\pi_b}(P) + \eta_k^{\pi_b}(P) +  \epsilon_k^{\pi_b}(P)) +  \alpha_k (V_{\hat{c}_{k}}^{\pi^*}(P) + \eta_k^{\pi^*}(P) +  \epsilon_k^{\pi^*}(P))\\
   & \overset{(a)}{\leq} (1-\alpha_k) (V_c^{\pi_b}(P)  + 2 \eta_k^{\pi_b}(P) + \epsilon_k^{\pi_b}(P)) +  \alpha_k (V_c^{\pi^*}(P) + 2 \eta_k^{\pi^*}(P) + \epsilon_k^{\pi^*}(P))\\
   &\overset{(b)}{\leq} (1-\alpha_k) (\bar{C}_b + 2 \eta_k^{\pi_b}(P) + \epsilon_k^{\pi_b}(P)) +  \alpha_k (\bar{C} + 2 \eta_k^{\pi^*}(P) + \epsilon_k^{\pi^*}(P)),
\end{align*}
where inequality $(a)$ is due to the good event that $c$ is within the confidence interval, $V_{\hat{c}_{k}}^{\pi}(P) - \eta_k^{\pi}(P) \leq V_c^{\pi}(P)$, for any $\pi$ and inequality $(b)$ is due to the fact that $V_c^{\pi_b}(P) = \bar{C}_b$, and $V_c^{\pi^*}(P) \leq \bar{C}$.

For  $(\tilde{\pi}_k, P)$ to be a feasible solution for \eqref{eq:modified-DOPE-optimization1}, it must be true that $V_{\bar{c}_{k}}^{\tilde{\pi}}(P) \leq \bar{C}$. Hence, it is sufficient to get an $\alpha_{k}$ such that 
\begin{align*}
 (1-\alpha_k) (\bar{C}_{b} + 2 \eta_k^{\pi_b}(P) + \epsilon_k^{\pi_b}(P)) + \alpha_k (\bar{C} + 2 \eta_k^{\pi^*}(P) +\epsilon_k^{\pi^*}(P))  \leq \bar{C}.
\end{align*}
This yields a sufficient condition \eqref{eq:alpha-condition1}. Note that $\alpha_{k}$ is non-negative because $2 \eta_k^{\pi_b}(P) + \epsilon^{\pi_b}_{k}(P) \leq \bar{C} - \bar{C}_{b}$ for $k \geq K_{o}$, as shown in the proof of Proposition \ref{prop:Feasibility}. This concludes the proof of Claim 1.

\textbf{Claim 2:}  $V_{\tilde{r}_k}^{\tilde{\pi}_{k}}(\tilde{P}_{k})  \leq V^{\pi^{*}}_{r}(P)$ if $b$ satisfies the sufficient condition
\begin{align}
\label{eq:b-condition1}
b \ge \frac{H}{\bar{C}-\bar{C}_b}.
\end{align}

\textit{Proof of Claim 2:} Selecting an $\alpha_{k}$ that satisfies the condition~\eqref{eq:alpha-condition1},  $(\tilde{\pi}, P)$ is a feasible solution of  \eqref{eq:modified-DOPE-optimization1}. Since  $(\tilde{\pi}_{k}, \tilde{P}_{k})$  is the optimal solution of  \eqref{eq:modified-DOPE-optimization1}, we have $V_{\tilde{r}_k}^{\tilde{\pi}_{k}}(\tilde{P}_{k})  \leq  V_{\tilde{r}_k}^{\tilde{\pi}}({P})$.  So, it is sufficient to find a $b$ such that $V_{\tilde{r}_k}^{\tilde{\pi}}({P}) \leq V^{\pi^{*}}_{r}(P)$. Using the linearity of the value function w.r.t. occupancy measure,  this is equivalent to $
 (1-\alpha_k) (V_{\hat{r}_k}^{\pi_b}(P) - 3b \eta_k^{\pi_b}(P) - b \epsilon_k^{\pi_b}(P)) + \alpha_k (V_{\hat{r}_k}^{\pi^*}(P) - 3b \eta_k^{\pi^*}(P) - b \epsilon_k^{\pi^*}(P)) \leq  V^{\pi^*}_r(P)$. 
 
 Since $V_{\hat{r}_k}^{\pi}(P) - b  \eta_k^{\pi}(P) \leq V_{\hat{r}_k}^{\pi}(P) -  \eta_k^{\pi}(P) \leq V_r^{\pi}(P)$ for any $\pi$ under the good event, it is sufficient if we find a $b$ such that $(1-\alpha_k) (V_r^{\pi_b}(P) - 2b \eta_k^{\pi_b}(P) - b \epsilon_k^{\pi_b}(P)) + \alpha_k (V_r^{\pi^*} - 2b \eta_k^{\pi^*}(P) - b \epsilon_k^{\pi^*}(P)) \leq  V^{\pi^*}_r(P) $.
This will yield the condition $
b \geq \frac{V_r^{\pi_b}(P) - V_r^{\pi^*}(P)}{[\epsilon_k^{\pi_b}(P) + 2 \eta_k^{\pi_b}(P)] + \frac{\alpha_k}{1-\alpha_k} [\epsilon_k^{\pi^*}(P) + 2 \eta_k^{\pi^*}(P)]}.$
Now, we choose $\alpha_{k}$ that satisfies the condition \eqref{eq:alpha-condition1} as, $
\frac{\alpha_k}{1-\alpha_k} = \frac{\bar{C}-\bar{C}_b - [\epsilon_k^{\pi_b}(P) + 2 \eta_k^{\pi_b}(P)] }{\epsilon_k^{\pi^*}(P)+2 \eta_k^{\pi^*}(P)}$. Using this in the previous inequality for $b$, we get the sufficient condition $b \ge \frac{V_r^{\pi_b}(P) - V_r^{\pi^*}(P)}{\bar{C}-\bar{C}_b}$.   Since $V_r^{\pi_b}(P) \leq H$ and $V_r^{\pi^*}(P) \geq 0$, we get the sufficient condition  \eqref{eq:b-condition1}. This concludes the proof of Claim 2. 

Now, let $b = \frac{H}{\bar{C}-\bar{C}_b}$. So, $\tilde{r}_k= \bar{r}_{k}$ and $(\tilde{\pi}_{k}, \tilde{P}_{k}) = (\pi_{k}, P_{k})$.  Hence, by Claim 2, we have $V_{\bar{r}_k}^{{\pi}_{k}}({P}_{k})  \leq V^{\pi^{*}}_{r}(P)$. Hence, we have the desired result.
\end{proof}

We now present the proof of Theorem~\ref{thm:main-regret-theorem}.

\begin{proof}[\textbf{Proof of Theorem \ref{thm:main-regret-theorem}}]
The regret for the DOPE algorithm after $K$ episodes can be written as,
\begin{equation}
\label{eq:regret-defn-st1}
R(K) = \sum^{K_{o}}_{k=1} (V^{\pi_k}_{r}(P) - V^{\pi^*}_{r}(P) ) + \sum_{k=K_{o}}^{K} (V_r^{\pi_k}(P) - V_r^{\pi^*}(P)).
\end{equation}
We will bound the first term  in \eqref{eq:regret-defn-st1} as
\begin{align}
    \label{eq:first-term-bd-st1}
    \sum^{K_{o}}_{k=1} (V^{\pi_k}_{r}(P) - V^{\pi^*}_{r}(P) )  \leq H K_{0} \leq   \tilde{\mathcal{O}}( \frac{S^{2} A H^{5}}{ (\bar{C} - \bar{C}_{b})^{2}}),
\end{align}
where we get the first inequality because  $(V^{\pi_k}_{r}(P) - V^{\pi^*}_{r}(P)) \leq H$, and the second inequality follows from the bound on $K_{o}$ in Proposition \ref{prop:Feasibility}. 

The second term in \eqref{eq:regret-defn-st1} can be bounded as
\begin{align} 
\sum_{k=K_{o}}^{K} &(V_r^{\pi_k}(P) - V_r^{\pi^*}(P)) =  \sum_{k=K_{o}}^{K} (V_r^{\pi_k}(P) - V_{\bar{r}_k}^{\pi_k}(P_k)) +  \sum_{k=K_{o}}^{K} (V_{\bar{r}_k}^{\pi_k}(P_k) - V_r^{\pi^*}(P)) \notag \\
    &\stackrel{(a)}{\leq} \sum_{k=K_{o}}^{K} (V_r^{\pi_k}(P) - V_{\bar{r}_k}^{\pi_k}(P_k)) 
    = \sum_{k=K_{o}}^{K} (V_r^{\pi_k}(P) - V_{\hat{r}_k}^{\pi_k}(P)) +  \sum_{k=K_{o}}^{K} (V_{\hat{r}_k}^{\pi_k}(P) - V_{\bar{r}_k}^{\pi_k}(P_k)) \notag \\
    & \stackrel{(b)}{ \leq}  \sum_{k=K_{o}}^{K} \eta_k^{\pi_k}(P) + \sum_{k=K_{o}}^{K}  (V_{\hat{r}_k}^{\pi_k}(P) - V_{\bar{r}_k}^{\pi_k}(P_k)) \notag \\
    &\stackrel{(c)}{\leq} \sum_{k=K_{o}}^{K} \eta_k^{\pi_k}(P) +\sum_{k=K_0}^{K} (V_{\bar{r}_k}^{\pi_k}(P) - V_{\bar{r}_k}^{\pi_k}(P_k)) +
    \frac{H}{\bar{C}-\bar{C}_b} \sum_{k=K_0}^{K} (3 \eta_k^{\pi_k}(P) + \epsilon_k^{\pi_k}(P)) \notag \\
     \label{eq:second-term-bd-st1} 
    &\stackrel{(d)}{\leq} \sum_{k=K_{o}}^{K} (V_{\bar{r}_k}^{\pi_k}(P) - V_{\bar{r}_k}^{\pi_k}(P_k)) + \tilde{O}({\frac{H^3}{\bar{C}-\bar{C}_b}} S\sqrt{AK}),
\end{align}
where $(a)$ is due to the fact that $V_{\bar{r}^k}^{\pi_k}(P_k) \leq V_r^{\pi^*}(P)$ from Lemma~\ref{lem:termII-decomp-Optimism}, $(b)$ is due to the fact that $|r_h(s,a) -  \hat{r}_{h,k}(s,a)| \leq  \beta_{h,k}^{l}(s,a)$ conditioned on the good event set $G$  (see Lemma \ref{lem:r-goodevent}),  $(c)$ follows from the definition of $\bar{r}_k$, and $(d)$ follows from Lemma \ref{lem:epsilon-k-sum-bound}. 

We will now bound the first term in \eqref{eq:second-term-bd-st1} as 
\begin{align}
\hspace{-1cm}&\sum_{k=K_0}^{K} (V_{\bar{r}_k}^{\pi_k}(P) - V_{\bar{r}_k}^{\pi_k}(P_k))\notag\\ &\stackrel{(e)}{=}
\sum_{k=K_0}^{K} \sum_{h=1}^H \mathbb{E} [\sum_{s'} (P_h - P_{h,k})(s'|s_{h,k},a_{h,k}) V_{\bar{r}_k,h+1}^{\pi_k}(s';P_k) | \pi_k, P, \mathcal{F}_{k-1} ] \nonumber\\
&\leq  \sum_{k=K_0}^{K} \sum_{h=1}^H \mathbb{E} [\sum_{s'} |(P_h - P_{h,k})(s'|s_{h,k},a_{h,k})| |V_{r,h+1}^{\pi_k}(s';P)| | \pi_k, P, \mathcal{F}_{k-1} ]~~+ \notag\\
&\sum_{k=K_0}^{K} \sum_{h=1}^H \mathbb{E} [|\sum_{s'} (P_h - P_{h,k})(s'|s_{h,k},a_{h,k}) (V_{\bar{r}_k,h+1}^{\pi_k}(s';P_k) - V_{r,h+1}^{\pi_k}(s';P))| | \pi_k, P, \mathcal{F}_{k-1} ] \label{eq:second-term-bd-st2}
\end{align}

Now, for the first term in~\eqref{eq:second-term-bd-st2},
\begin{align}
          \sum_{k=K_0}^{K} \sum_{h=1}^H \mathbb{E} [\sum_{s'} |(P_h - P_{h,k})(s'|s_{h,k},a_{h,k})| & |V_{r,h+1}^{\pi_k}(s';P)| | \pi_k, P, \mathcal{F}_{k-1} ] \notag\\
          &\leq \sum_{k=K_0}^{K} \epsilon_k^{\pi_k}(P) \stackrel{(f)}{\leq} \tilde{O}(S\sqrt{AH^4K'}) \label{eq:first-term-bd-st2},
\end{align}
      
where $(e)$ is obtained from the value difference lemma (Lemma \ref{lem:value-difference}), and $(f)$ is from Lemma \ref{lem:epsilon-k-sum-bound}.

In order to  now bound the second term in \eqref{eq:second-term-bd-st2}, we proceed in similar lines to the proof of Lemma $32$ from \cite{efroni2020exploration}. 

Consider, 
\begin{align}
  \sum_{k=K_0}^{K} \sum_{h=1}^H   & \mathbb{E} [|\sum_{s'} (P_h - P_{h,k})(s'|s_{h,k},a_{h,k}) (V_{\bar{r}_k,h+1}^{\pi_k}(s';P_k) - V_{r,h+1}^{\pi_k}(s';P))| | \pi_k, P, \mathcal{F}_{k-1}] \notag \\
    &=  \sum_{k,h,s,a} w_h^{\pi_k}(s,a;P) \sum_{s'} (P_h - P_{h,k})(s'|s_{h,k},a_{h,k}) ~  (V_{\bar{r}_k,h+1}^{\pi_k}(s';P_k) - V_{r,h+1}^{\pi_k}(s';P)) \notag \\
    & \leq \underbrace{\sum_{k,h,s,a} w_h^{\pi_k}(s,a;P) \sum_{s'} \frac{\sqrt{P_h(s'|s,a)}}{\sqrt{n_{h,k}(s,a) \vee 1}} |V_{\bar{r}_k,h+1}^{\pi_k}(s';P_k) - V_{r,h+1}^{\pi_k}(s';P)|}_{(A)} \notag \\
    \label{eq:second-term-bd-st3} 
    & \hspace{1cm}+ \underbrace{\sum_{k,h,s,a} w_h^{\pi_k}(s,a;P)  \frac{1}{n_{h,k}(s,a) \vee 1}  |V_{\bar{r}_k,h+1}^{\pi_k}(s';P_k) - V_{r,h+1}^{\pi_k}(s';P)| }_{(B)},
\end{align} 
where the last inequality is obtained from Lemma~\ref{lem: confidence}.  We will bound the term $(B)$ in \eqref{eq:second-term-bd-st3}  as
\begin{align}
\label{eq:second-term-bd-st4-B}
(B) \leq \frac{H^3SL}{\bar{C}-\bar{C}_b} \sum_{k,h,s,a} w_h^{\pi_k}(s,a;P)  \frac{1}{n_{h,k}(s,a) \vee 1} \leq \frac{H^5S^2AL}{\bar{C}-\bar{C}_b},
\end{align}
where the first inequality is from  bounding $|V_{\bar{r}_k}^{\pi_k}(P_k) - V_r^{\pi_k}(P)|$ by $\frac{H^3SL}{\bar{C}-\bar{C}_b}$. This is obtained by noting that $V_{\bar{r}_k}^{\pi_k}(P_k) = V_{\hat{r}_k}^{\pi_k}(P_k) - \frac{H}{\bar{C}-\bar{C}_b} \eta_k^{\pi_k}(P_k) - \frac{H^2}{\bar{C}-\bar{C}_b} \epsilon_k^{\pi_k}(P_k) \leq V_{\hat{r}_k}^{\pi_k}(P_k) + \frac{H}{\bar{C}-\bar{C}_b} \mathbb{E}[\sum_{h=1}^H \beta^l_{h,k}(s_{h,k},a_{h,k})|\pi_k,P_k] + \frac{H^2}{\bar{C}-\bar{C}_b} \mathbb{E}[\sum_{h=1}^H \sum_{s'} \beta^p_{h,k}(s_{h,k},a_{h,k},s')|\pi_k,P_k] \leq  \frac{H^3SL}{\bar{C}-\bar{C}_b}$, since $\sum_{h=1}^H \sum_{s'}\beta^p_{h,k}(s_{h,k},a_{h,k},s') \leq HSL$, from the definition. The second inequality is from  Lemma \ref{lem:epsilon-k-sum-bound}.

We now bound the term $(A)$ in \eqref{eq:second-term-bd-st3} as follows.  
\begin{align}
    &(A) \overset{(a)}{\leq}  \sum_{k,h,s,a} w_h^{\pi_k}(s,a;P) \frac{\sqrt{S \sum_{s'} P_h(s'|s,a) (V_{\bar{r}_k,h+1}^{\pi_k}(s';P_k) - V_{r,h+1}^{\pi_k}(s';P))^2 }} {\sqrt{n_{h,k}(s,a) \vee 1}} \notag \\
    & \overset{(b)}{\leq} \! \!\sqrt{S} \!\! \left(\! \sum_{k, h,s,a} \! \frac{w_h^{\pi_k}(s,a;P)}{n_{h,k}(s,a) \vee 1}\!\right)^{\!\!\!\!1/2\!\!} \! \! \!
    \left(\!\sum_{k,h,s,a,s'} \!\!\!\!\!\! w_h^{\pi_k}(s,a;P) P_h(s'|s,a) (V_{\bar{r}_k,h+1}^{\pi_k}(s';P_k) \!-\! V_{r,h+1}^{\pi_k}(s';P))^2 \! \! \right)^{\!\!\!\!1/2\!\!} \notag \\
    &\overset{(c)}{=} \sqrt{S} \left( \sum_{k,h,s,a}  \frac{w_h^{\pi_k}(s,a;P)}{n_{h,k}(s,a) \vee 1}\right)^{\!\!\!\!1/2} \left(\sum_{k, h,s',a}  w_{h+1}^{\pi_k}(s',a;P) (V_{\bar{r}_k,h+1}^{\pi_k}(s';P_k) - V_{r,h+1}^{\pi_k}(s';P))^2 \right)^{\!\!\!\!1/2} \notag\\
    & \overset{(d)}{\leq} \sqrt{S} \sqrt{SAH^2} \left(\sum_{k, h,s,a}  w_{h+1}^{\pi_k}(s,a;P)  (V_{r,h+1}^{\pi_k}(s;P)-V_{\bar{r}_k,h+1}^{\pi_k}(s;P_k))^{2} \right)^{\!\!\!\!1/2} \notag
    \end{align}
    
    \begin{align}
    & \overset{(e)}{\leq} \sqrt{S} \sqrt{SAH^2} \frac{H^3SL}{\bar{C}-\bar{C}_b} \left(\sum_{k,h,s,a}  w_{h+1}^{\pi_k}(s,a;P)  ( V_{r,h+1}^{\pi_k}(s;P)-V_{\bar{r}_k,h+1}^{\pi_k}(s;P_k)) \right)^{1/2} \notag\\
    &\overset{(f)}{\leq}  \frac{S^2H^{4.5}\sqrt{A}L}{\bar{C}-\bar{C}_b} \left(\sum_k (V_r^{\pi_k}(s_1;P) - V_{\bar{r}_k}^{\pi_k}(s_1;P_k))  \right. \notag \\
    & \hspace{10mm}+ \left.\sum_{k,h,s,a} w_h^{\pi_k}(s,a;P) |\langle (P_h-P_{h,k})(.|s,a),(V_{\bar{r}^k,h+1}^{\pi_k}(\cdot;P_k) - V_{r,h+1}^{\pi_k}(\cdot;P)) \rangle|\right)^{1/2} \notag\\
    \label{eq:second-term-bd-st4-A}
    &\overset{(g)}{\leq} \frac{S^2H^{4.5}\sqrt{A}L}{\bar{C}-\bar{C}_b} \left(\sum_k (V_r^{\pi_k}(s_1;P) - V_{\bar{r}^k}^{\pi_k}(s_1;P_k))\right)^{1/2}  \notag\\
    & +\frac{S^2H^{4.5}\sqrt{A}L}{\bar{C}-\bar{C}_b} \left(\sum_{k,h,s,a} \!\!\! w_h^{\pi_k}(s,a;P) |\langle (P_h-P_{h,k})(.|s,a),(V_{\bar{r}^k,h+1}^{\pi_k}(.;P_k) - V_{r,h+1}^{\pi_k}(.;P)) \rangle| \right)^{\!\!\!\!1/2}.
\end{align}
Here, $(a)$ is obtained by Jensen's inequality, $(b)$ is by cauchy schwartz inequality, $(c)$ is  from the property of the occupancy measure, i.e., $\sum_{s,a} P_h(s'|s,a) w_h(s,a,P) = \sum_a w_{h+1}(s',a,P)$,  $(d)$ is obtained from Lemma~\ref{lem:nk-sum-lemma-2}. To get $(e)$, we use the result from Lemma \ref{lem: ineqA} that $V_{\bar{r}_k,h+1}^{\pi_k}(P_k) \leq V_{r,h+1}^{\pi_k}(P)$, and hence obtain, $(V_{\bar{r}_k,h+1}^{\pi_k}(P_k)-V_{r,h+1}^{\pi_k}(P)(s))^2 \leq \frac{H^3SL}{\bar{C}-\bar{C}_b} (V_{r,h+1}^{\pi_k}(P)-V_{\bar{r}_k,h+1}^{\pi_k}(P_k)(s))$. We prove $(e)$ from Lemma $33$ of~\cite{efroni2020exploration}. The step is to get $(f)$ is more involved and we prove it separately  in Lemma \ref{lem: ineqE}, following a similar result from Lemma $33$ of \cite{efroni2020exploration}. The inequality $(g)$ holds from the fact that $\sqrt{a+b} \leq \sqrt{a}+\sqrt{b}$.

Using the above obtained bounds on $(A)$ and $(B)$ in \eqref{eq:second-term-bd-st3}, we get
\begin{align*}
    &\sum_{k,h,s,a} w_h^{\pi_k}(s,a;P) |\langle (P_h-P_{h,k})(.|s,a), (V_{\bar{r}_k,h+1}^{\pi_k}(\cdot;P_k)- V_{r,h+1}^{\pi_k}(\cdot;P)) \rangle| \\
    &\leq \frac{H^5S^2AL}{\bar{C}-\bar{C}_b} +  \frac{S^2H^{4.5}\sqrt{A}L}{\bar{C}-\bar{C}_b} \left(\sum_k (V_r^{\pi_k}(s_1;P) - V_{\bar{r}_k}^{\pi_k}(s_1;P_k))\right)^{1/2} \\
    &\hspace{2mm}+\frac{S^2H^{4.5}\sqrt{A}L}{\bar{C}-\bar{C}_b} \left(\sum_{k,h,s,a} \!\!\! w_h^{\pi_k}(s,a;P) |\langle (P_h-P_{h,k})(.|s,a)(V_{\bar{r}_k,h+1}^{\pi_k}(\cdot;P_k) - V_{r,h+1}^{\pi_k}(\cdot;P)) \rangle| \right)^{\!\!\!1/2}.
\end{align*}

Let $X = \sum_{k,h,s,a} w_h^{\pi_k}(s,a,P) |\langle (P_h-P_{h,k})(.|s,a) (V_{\bar{r}_k,h+1}^{\pi_k}(\cdot;P_k)- V_{r,h+1}^{\pi_k}(\cdot;P) \rangle| $. Then, the above bound takes the form $0 \leq X \leq a + b \sqrt{X}$,
where $a = \frac{H^5S^2AL}{\bar{C}-\bar{C}_b} +  \frac{S^2H^{4.5}\sqrt{A}L}{\bar{C}-\bar{C}_b} \left(\sum_k (V_r^{\pi_k}(s_1;P) - V_{\bar{r}_k}^{\pi_k}(s_1;P_k))\right)^{1/2} $, and $b = \frac{S^2H^{4.5}\sqrt{A}L}{\bar{C}-\bar{C}_b}$.

Now, using the fact that, if $0 \leq X \leq a + b\sqrt{X}$, then $X \leq a + b^2$ (Lemma $38$ from \cite{efroni2020exploration}), we can obtain the bound
\begin{align}
    &\sum_{k,h,s,a} w_h^{\pi_k}(s,a;P) |\langle (P_h-P_{h,k})(.|s,a), (V_{\bar{r}_k,h+1}^{\pi_k}(\cdot;P_k)- V_{r,h+1}^{\pi_k}(\cdot;P)) \rangle| \notag \\
    \label{eq:second-term-bd-st5}
    & \leq \frac{4H^5S^2AL}{\bar{C}-\bar{C}_b} +
    \frac{4S^2H^{4.5}\sqrt{A}L}{\bar{C}-\bar{C}_b} \left(\sum_k (V_r^{\pi_k}(s_1;P) - V_{\bar{r}_k}^{\pi_k}(s_1;P_k))\right)^{\!\!1/2\!\!} + \frac{S^4H^9AL}{(\bar{C}-\bar{C}_b)^2}.
\end{align}

Substituting the above bound and~\eqref{eq:first-term-bd-st2} in \eqref{eq:second-term-bd-st2}, we obtain,
\begin{align*}
    &\sum_{k} V_{\bar{r}_k}^{\pi_k}(P) - V_{\bar{r}_k}^{\pi_k}(P_k)
    \leq \notag \\
    &S \sqrt{AH^4K} + \frac{H^5S^2AL}{\bar{C}-\bar{C}_b} +
    \frac{S^2H^{4.5}\sqrt{A}L}{\bar{C}-\bar{C}_b} \left(\!\!\sum_k (V_r^{\pi_k}(s_1;P) - V_{\bar{r}_k}^{\pi_k}(s_1;P_k))\!\!\right)^{\!\!\!1/2\!\!\!} + \frac{S^4H^9AL}{(\bar{C}-\bar{C}_b)^2}.
\end{align*}

Using the above bound in \eqref{eq:second-term-bd-st1}, we obtain,
\begin{align}
    &\sum_{k} V_{r}^{\pi_k}(P) - V_{\bar{r}_k}^{\pi_k}(P_k) 
    \leq   S \sqrt{AH^4K} + \frac{H^5S^2AL}{\bar{C}-\bar{C}_b}  \notag \\
    \label{eq:second-term-bd-st7}
    &\hspace{10mm}+\frac{S^2H^{4.5}\sqrt{A}L}{\bar{C}-\bar{C}_b} \left(\!\!\sum_k (V_r^{\pi_k}(s_1;P) - V_{\bar{r}_k}^{\pi_k}(s_1;P_k))\!\!\right)^{\!\!1/2\!\!} +\frac{S^4H^9AL}{(\bar{C}-\bar{C}_b)^2} +  \frac{S \sqrt{AH^6K}}{\bar{C}-\bar{C}_b} .
\end{align}
The left hand side of the above inequality is non-negative, since $V_{\bar{r}_k}^{\pi_k}(P_k) \leq V_r^{\pi^*}(P)$, from lemma~\ref{lem:termII-decomp-Optimism}, and $V_r^{\pi^*}(P) \leq V_r^{\pi_k}(P)$, since $\pi^*$ is the optimal policy on $P$. This equation is again of the form $0 \leq X \leq a + b \sqrt{X}$, where $X = \sum_{k} V_{r}^{\pi_k}(P) - V_{\bar{r}_k}^{\pi_k}(P_k) $. Using the same result we used to get \eqref{eq:second-term-bd-st5}, we deduce that  $X \leq a + b^2$, and hence,
\begin{align*}
    \sum_{k} V_{r}^{\pi_k}(P) - V_{\bar{r}_k}^{\pi_k}(P_k) &\leq  \frac{S \sqrt{AH^6K}}{\bar{C}-\bar{C}_b}   + S \sqrt{AH^4K} + \frac{H^5S^2AL}{\bar{C}-\bar{C}_b} +
    \frac{S^4H^{9}AL}{(\bar{C}-\bar{C}_b)^2}  + \frac{S^4H^9AL}{(\bar{C}-\bar{C}_b)^2}\\
    & \leq \tilde{O} (\frac{H}{\bar{C}-\bar{C}_b} S\sqrt{AH^4K}).
\end{align*}
and hence, from~\eqref{eq:second-term-bd-st1}, $\sum_k V_r^{\pi_k}(P) - V_r^{\pi^*}(P) \leq \tilde{O} (\frac{H}{\bar{C}-\bar{C}_b} S\sqrt{AH^4K'})$.

Moreover, from proposition~(\ref{prop:safety}), we have that $\pi_{k} \in \Pi_{\textnormal{safe}}$ for all $k \in [K]$, with probability $1-5 \delta$.
\end{proof}

\begin{lemma}[Bonus optimism] \label{lem: ineqA}
For any $(s,a,h,k)$, conditioned on good event, we have that $r_h(s_h,a_h) - \bar{r}_{h,k}(s_h,a_h) + \left(<P_h(.|s_h,a_h) - P_{h,k}(.|s_1,a_1),  V_{r,h+1}^{\pi_k}(.;P)> \right)\ge 0$, and, for any $\pi,s,h,k$, it holds that $V_{\bar{r}_k,h}^{\pi}(s;P_k) \leq V_{r,h}^{\pi}(s;P)$.
\end{lemma}

\begin{proof}
Consider
\begin{align*}
    &\bar{r}_{h,k}(s_h,a_h) - r_h(s_h,a_h) + <P_{h,k}(.|s_h,a_h) - P_h(.|s_h,a_h),  V_{r,h+1}^{\pi_k}(.;P)> \\
    &\leq \sum_{s'} |(P_{h,k}-P_h)(s'|s,a)||V_{r,h+1}^{\pi_k}(s')| - \frac{H}{\bar{C}-\bar{C}_b} \beta_{h,k}^r(s_h,a_h) - \frac{H^2}{\bar{C}-\bar{C}_b} \bar{\beta}^p_{h,k}(s_1,a_1)\\
    &\leq H \bar{\beta}^p_{h,k}(s_h,a_h) - \frac{H^2}{\bar{C}-\bar{C}_b} \bar{\beta}_{h,k}^p(s_h,a_h)-\frac{H}{\bar{C}-\bar{C}_b} \beta_{h,k}^r(s_h,a_h)\\
    &\leq 0,
\end{align*}
where the last inequality is due to the fact that $\bar{C}-\bar{C}_b \leq H$. 

Similarly, by value difference lemma~\ref{lem:value-difference}, we have,
\begin{align*}
&V_{\bar{r}_k,h}^{\pi}(s;P_k) - V_{r,h}^{\pi}(s;P) \\
&= \mathbb{E} \left[ \sum_{h'=h}^H \bar{r}_{h',k}(s_{h'},a_{h'}) - r_{h'}(s_{h'},a_{h'}) - (P_{h'} - P_{h',k})(.|s_{h'},a_{h'}) V_{h'+1,r}^{\pi} (.;P) |s_h = s, \pi, P_k\right] \\
&\leq 0,
\end{align*}
where the last inequality is obtained just earlier.
\end{proof}

\begin{lemma}\label{lem: ineqE}
We prove inequality $(e)$ in bounding term $(i)$, i.e., we prove that,
\begin{align*}
\sum_{h,s,a} &w_{h+1}^{\pi_k}(s,a,P)( V_{r,h+1}^{\pi_k}(s;P)-V_{\bar{r}_k,h+1}^{\pi_k}(s;P_k)) \leq H(V_r^{\pi_k}(s_1;P) - V_{\bar{r}_k}^{\pi_k}(s_1;P_k))  \\ & +H \sum_{h,s,a} w_h^{\pi_k}(s,a,P)|<(P_h-P_{h,k})(.|s,a),(V_{\bar{r}_k,h+1}^{\pi_k}(\cdot;P_k) - V_{r,h+1}^{\pi_k}(\cdot;P))>|.
\end{align*}
\end{lemma}
\begin{proof}
Let us start with,
\begin{align*}
    &V_{r,1}^{\pi_k}(s_1;P) - V_{\bar{r}_k,1}^{\pi_k}(s_1;P_k) \\
    &= \mathbb{E}  \left[ V_{r,1}^{\pi_k}(s_1;P) - r_1(s_1,a_1) - \langle P_1(.|s_1,a_1), V_{\bar{r}_k,2}^{\pi_k}(.;P_k)\rangle | \pi_k,P\right] +\\
    & \hspace{10mm}+ \mathbb{E}  \left[ r_1(s_1,a_1) + \langle P_1(.|s_1,a_1), V_{\bar{r}_k,2}^{\pi_k}(.;P_k) - V_{\bar{r}_k,1}^{\pi_k}(s_1;P_k)\rangle|  \pi_k,P\right]\\
    &= \mathbb{E} \left[\langle P_1(.|s_1,a_1), (V_{r,2}^{\pi_k}(.;P) - V_{\bar{r}_k,2}^{\pi_k}(.;P_k))\rangle|  \pi_k,P\right] \\
    & +\mathbb{E}  \left[ r_1(s_1,a_1) + \langle P_1(.|s_1,a_1) V_{\bar{r}_k,2}^{\pi_k}(.;P_k) \rangle - \bar{r}_{1,k}(s_1,a_1) - \langle P_{1,k}(.|s_1,a_1) V_{\bar{r}_k,2}^{\pi_k}(s_1;P_k)\rangle| \pi_k,P\right]
    \end{align*}
    
    \begin{align*}
    &= \mathbb{E} \left[\langle P_1(.|s_1,a_1), (V_{r,2}^{\pi_k}(.;P) - V_{\bar{r}_{k},2}^{\pi_k}(.;P_k))\rangle| P, \pi_k\right]+\\
    &\hspace{10mm} \mathbb{E} \left[\langle (P_1-P_{1,k})(.|s_1,a_1), (V_{\bar{r}_k,2}^{\pi_k}(.;P_k) - V_{r,2}^{\pi_k}(.;P))\rangle \right]+\\
    &\hspace{10mm}\mathbb{E}  \left[ r_1(s_1,a_1) - \bar{r}_{1,k}(s_1,a_1) + \langle (P_1-P_{1,k})(.|s_1,a_1),  V_{r,2}^{\pi_k}(s_1;P)\rangle|  \pi_k,P\right]\\
    & \overset{(a)}{\ge} \mathbb{E} \left[V_{r,2}^{\pi_k}(.;P) - V_{\bar{r}_{k},2}^{\pi_k}(.;P_k)|s_1, P, \pi_k\right]+ \\
    & \hspace{10mm}\mathbb{E} \left[\langle (P_1-P_{1,k})(.|s_1,a_1), V_{\bar{r}_{k},2}^{\pi_k}(.;P_k) - V_{r,2}^{\pi_k}(.;P)\rangle |\pi_k,P \right],
\end{align*}
where inequality $(a)$ is from Lemma~\ref{lem: ineqA}.
Iterating this relation, for $h$ times, we get,
\begin{align*}
&V_{r,1}^{\pi_k}(P) - V_{\bar{r}_k,1}^{\pi_k}(P_k) \\
&\ge \mathbb{E} \left[ V_{r,h}^{\pi_k}(.;P) - V_{\bar{r}_k,h}^{\pi_k}(.;P_k)|s_1, P, \pi_k\right]+\\
&\hspace{10mm}\sum_{h'=1}^{h-1}\mathbb{E} \left[<(P_{h'}-P_{h',k})(.|s_{h'},a_{h'}), V_{\bar{r}_k,h'+1}^{\pi_k}(.;P_k) - V_{r,h'+1}^{\pi_k}(.;P)|s_1,\pi_k,P \right].
\end{align*}

Summing this relation for $h \in \{2,\dots,H\}$, we get,
\begin{align*}
    &H\left[V_{r,1}^{\pi_k}(P) - V_{\bar{r}_k,1}^{\pi_k}(P_k)\right] - \\
    &\sum_{h=2}^H \sum_{h'=1}^{h-1}\mathbb{E} \left[<(P_{h'}-P_{h',k})(.|s_{h'},a_{h'}), V_{\bar{r}_k,h'+1}^{\pi_k}(.;P_k) - V_{r,h'+1}^{\pi_k}(.;P)> |s_1,\pi_k,P\right] \\
    &\ge \sum_{h=2}^H \mathbb{E} \left[ V_{r,h}^{\pi_k}(.;P) - V_{\bar{r}_k,h}^{\pi_k}(.;P_k)|s_1, P, \pi_k\right].
\end{align*}
Hence, we obtain,
\begin{align*}
& \sum_{h,s,a} w_{h+1}^{\pi_k}(s,a,P) (V_{r,h+1}^{\pi_k}(s;P) - V_{\bar{r}_k,h+1}^{\pi_k}(s;P_k))=\sum_{h=2}^H \mathbb{E} \left[ V_{r,h}^{\pi_k}(.;P) - V_{\bar{r}_k,h}^{\pi_k}(.;P_k)|s_1 , P, \pi_k\right] \\
    &\leq H\left[V_{r,1}^{\pi_k}(P) - V_{\bar{r}_k,1}^{\pi_k}(P_k)\right] \\
    & \hspace{10mm}+ \sum_{h=2}^H \sum_{h'=1}^{h-1}\mathbb{E} \left[-<(P_{h'}-P_{h',k})(.|s_{h'},a_{h'}), V_{\bar{r}_k,h'+1}^{\pi_k}(.;P_k) - V_{r,h'+1}^{\pi_k}(.;P)> |s_1,\pi_k,P\right]\\
    &\leq H\left[V_{r,1}^{\pi_k}(P) - V_{\bar{r}_k,1}^{\pi_k}(P_k)\right] \\
    & \hspace{10mm}+ \sum_{h=2}^H \sum_{h'=1}^{H}\mathbb{E} \left[|<(P_{h'}-P_{h',k})(.|s_{h'},a_{h'}), V_{\bar{r}_k,h'+1}^{\pi_k}(.;P_k) - V_{r,h'+1}^{\pi_k}(.;P)>||s_1,\pi,P\right]\\
    & \leq H\left[V_{r,1}^{\pi_k}(P) - V_{\bar{r}_k,1}^{\pi_k}(P_k)\right] \\
    & \hspace{10mm}+ H \sum_{h=1}^{H}\mathbb{E} \left[|<(P_h-P_{h,k})(.|s_{h},a_{h}), V_{\bar{r}_k,h+1}^{\pi_k}(.;P_k) - V_{r,h+1}^{\pi_k}(.;P)>||s_1,\pi_k,P\right].
\end{align*}
\end{proof}

\section{Detailed Description of Experiment Environments and Algorithm  Implementation}\label{sec:experiments}

\subsection{Experiment Environments}

\textbf{Factored CMDP environment:} The factored CMDP is represented in Fig.~\ref{fig: SimpleCMDPEnv}.  The state space is $\S = \{1,2,3\}$, and the action space is $\A = \{1,2\}$, where action $1$ corresponds to moving one step to the right and $2$ corresponds to staying put. Objective cost is $r(s,1) = 0, \forall s \in \S$, and $r(s,2) = s, \forall s \in \S$. The constraint cost is, $c(s,1) = 0, \forall s \in \S$, and $c(s,2) = 1, \forall s \in \S$. The probability transition matrix under action $1$ is,$
\begin{pmatrix}
0 & 1 & 0 \\
0 & 0 & 1 \\
1 & 0 & 0
\end{pmatrix}
$, and under action $2$, it is $
\begin{pmatrix}
1 & 0 & 0 \\
0 & 1 & 0 \\
0 & 0 & 1
\end{pmatrix}
$.

\begin{figure} 
\centering
\begin{tikzpicture}[
roundnode/.style={circle, draw=red!60, fill=green!5, very thick, minimum size=7mm}
]
\node[roundnode]    (uppercircle) {1};
\node[roundnode]  (rightcircle)   [below right= 1.5cm of uppercircle] {2};
\node[roundnode]    (leftcircle)   [below left=1.5cm of uppercircle] {3};

\draw[-{Latex[scale=1]}] (uppercircle.south) -- node[midway,sloped,above,xslant=0,text width = 1.3cm]
{\scriptsize $r(1,1)=0$ $c(1,1)=0$}  
(rightcircle.north);

\draw[-{Latex[scale=1]}] (rightcircle.west) --
node[right,sloped,below,xslant=0,text width = 1.6cm] {\scriptsize $r(2,1)=0$ $c(2,1)=0$} 
(leftcircle.east);

\draw[-{Latex[scale=1]}] (leftcircle.north) --
node[right,sloped,above,xslant=0,text width = 1.6cm] {\scriptsize $r(3,1)=0$ $c(3,1)=0$} 
(uppercircle.south);

\draw[dashed,->] (uppercircle.north) .. controls +(up:10mm) and +(right:10mm) .. 
node[align=right,above,xslant=0,text width = 1.6cm] {\scriptsize $r(1,2)=1$ $c(1,2)=1$} 
(uppercircle.north east);

\draw[dashed,->] (rightcircle.north east) .. controls +(up:10mm) and +(right:10mm) .. 
node[align=left,right,xslant=0,text width = 1.6cm] {\scriptsize $r(2,2)=2$ $c(2,2)=1$} 
(rightcircle.south east);

\draw[dashed,->] (leftcircle.south west) .. controls +(down:5mm) and +(left:12mm) .. 
node[align=right,left,xslant=0,text width = 1.6cm] {\scriptsize $r(2,2)=2$ $c(2,2)=1$} 
(leftcircle.north west);

\end{tikzpicture}
\caption{Illustrating the Factored CMDP environment. Environment transitions to next state with action $1$, and stays put with action $2$.}
\label{fig: SimpleCMDPEnv}
\end{figure}
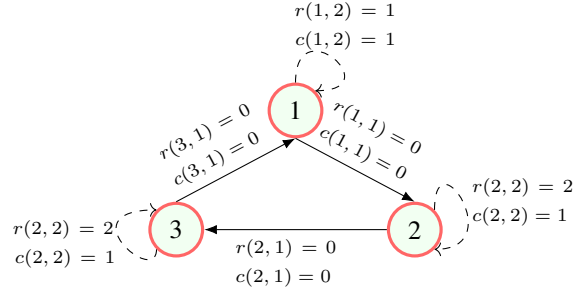

\textbf{Media Streaming Environment:} Here, we model the media streaming control  from a wireless base station. The base station provides two types of service to a device, a fast service and a slow service. The packets received are stored in a media buffer at the device. The goal is to minimize the cost of having an empty buffer (which may result in stalling of the video), while keeping the utilization of fast service below certain level.

We denote by $A_h$, the number of incoming packets into the buffer, and by $B_h$, the number of packets leaving the buffer. The state of the environment, denoted as $s_h$ in $h^{th}$ step,  is the media buffer length. It evolves as, $s_{h+1} = \min\{\max(0, s_h + A_h - B_h),N\}$. We  consider $N = 20$ as the maximum buffer length in our experiment.
The action space is $\mathcal{A} = \{1,2\}$, i.e., the action is to use either fast server $1$ or slow server $2$. We assume that the service rates of the servers have independent Bernoulli distributions, with parameters $\mu_1 = 0.9$, and $\mu_2 = 0.1$, where $\mu_1$ corresponds to the fast service.  The media playback at the device is also Bernoulli with parameter $\gamma$. Hence, $A_h$ is a random variable with mean either $\mu_1$ or $\mu_2$ depending on the action taken, and $B_h$ is a random variable with mean $\gamma$. These components constitute the unknown transition dynamics of our environment.

The objective cost is $r(s,a) = \mathbbm{1}\{s = 0\}$. i.e., it has a value of $1$, when the buffer hits zero, and is zero everywhere else. Our constraint cost is $c(s,a) = \mathbbm{1}\{a = 1\}$, i.e., there is a constraint cost of $1$ when the fast service is used, and is zero otherwise. We then constrain the expected number of times the fast service is used to $\bar{C} = {H}/{2}$, in a horizon of length $H=10$.

\textbf{Inventory Control Environment:} We consider a single product inventory control problem \cite{bertsekas2000dynamic}. Our environment evolves according to a finite horizon CMDP, with horizon length $H = 7$, where each time step $h \in [H]$ represents a day of the week. In this problem, our goal is to maximize the expected total revenue over a week, while keeping the expected total costs in that week below a certain level.  We do not backlog the demands.

The storage has a maximum capacity $N=6$, which means it can store a maximum of $6$ items. We denote by $s_h$, the state of the environment, as the amount of inventory available at $h^{th}$ day. The action $a_h$ is the amount of inventory the agent purchases such that the inventory does not overflow. Thus, the action space $\mathcal{A}_s \in \{0,\dots, N - s\},$ for the state $s$. The exogenous demand is represented by $d_h$, which is a random variable representing the stochastic demand for the inventory on the $h^{th}$ day. We assume $d_h$ to be in $\{0,\cdots, N\}$ with distribution $[0.3, 0.2, 0.2, 0.05,0.05]$.  If the demand is higher than the inventory and supply, the excess demand will not be met. The state evolution then follows as $ s_{h+1} = \max \{ 0, s_h + a_h - d_h \}$. 

We define the rewards and costs as follows. The  revenue is generated as, $f(s, a, s') = 8 (s+a-s')$, when $s'>0$, and is $0$ otherwise. The reward obtained in state $(s,a)$ is then the expected revenue over all next states $s'$, $r(s,a) = \mathbb{E}[f(s,a,s')]$. The cost associated with the inventory has two components. Firstly, there is a purchase cost when the inventory is brought in, which is a fixed cost of $4$ units, plus a variable cost of $2a$, which increases with the amount of purchase. Secondly, we also have a non decreasing holding cost $s$, for storing the inventory. Hence, the cost in $(s,a)$ is $c(s,a) = 4 + 2a + s$. We normalize the rewards and costs to be in the range $[0,1]$. Our goal is to maximize the expected total revenue over a week $(H = 7)$, while keeping the expected total costs in that week below a threshold $\bar{C} = {H}/{2}$.

\subsection{Details of the Implementation}
We now describe the algorithms described in the introduction.

\textbf{AlwaysSafe:} AlwaysSafe~\cite{simao2021alwayssafe} shows empirical results that only depict the expected cost of various policies deduced from their linear programs versus the optimal expected cost. We note that this comparison is not a reasonable measure of regret, since cumulative regret can be linear even though the expected costs are close.

We consider a factored CMDP environment described previously. For implementing this algorithm, in each episode, we solve the LP4 linear program described in~\cite{simao2021alwayssafe}, based on the observations. For solving LP4, one also needs an abstract CMDP as described in Section $3$ of~\cite{simao2021alwayssafe}. We follow their description to construct such a model for the factored CMDP. The confidence intervals for AlwaysSafe algorithm are same as the ones for OptCMDP algorithm from~\cite{efroni2020exploration} and of DOPE. We notice that the regret of Always safe algorithms indeed grow at a linear rate.

\textbf{OptPessLP:} We implement Algorithm $1$ from~\cite{liu2021learning}.  We choose the baseline policy by solving the corresponding CMDP with a  more conservative constraint. For the factored CMDP and Media Control environment, we choose the constraint as $0.1 \bar{C}$, and solve the MDP to obtain $\pi_b$ and the corresponding cost $\bar{C}_b$. Similarly, for the inventory control environment, we choose the constraint as $0.2\bar{C}$. These choices of $\pi_b$ are the same for both OptPessLP and DOPE, for a fair comparison. We play $\pi_b$ when the condition in Equation $9$ from~\cite{liu2021learning} is met, otherwise we choose the maximizing policy from their linear program. We choose the confidence intervals as specified in their work, without any scaling. Despite the results suggested by theory, we notice that its empirical performance is very poor in every environment we consider.

\textbf{OptCMDP:} We implement Algorithm $1$ from~\cite{efroni2020exploration}. This algorithm solves the linear program that minimizes the objective cost with optimism in the model. Since this algorithm does not consider zero-violation setting,  we expect to see constraint violation of the same order as the regret. We use the  confidence intervals  as specified in their work, without any scaling.

\textbf{DOPE:} We implement Algorithm $1$ from our work. The choice of baseline policy $\pi_b$ is exactly the same as that for OptPessLP for each environment. $\pi_b$ is played until $K_0$ episodes, as provided by Proposition~\ref{prop:Feasibility}.
Then, the algorithm solves the linear program given by Equation~\eqref{eq:DOPE-optimization} to obtain $\pi_k$ in episode $k$.

The details for the linear program formulations are given in Appendices~\ref{sec:lp-cmdp} and~\ref{sec:extended-lp-cmdp}. DOPE, AlwaysSafe and OptCMDP algorithms use the Extended LP formulations, while OptPessLP uses the regular LP formulation.

For each environment, each of these algorithms are run for $20$ random seeds, and are averaged to obtain the regret plots in the figures.

\subsection{Experiment Results for Inventory Control Environment}
We show the performance of our DOPE algorithm in Inventory Control Environment.  As before, we compare it against the OptCMDP Algorithm~1 in~\cite{efroni2020exploration}, and and OptPess-LP algorithm from~\cite{liu2021efficient}.
Also, we choose the optimal policy from a conservative constrained problem (with a stricter constraint) as the baseline policy. We use $\bar{C}_b = 0.1 \bar{C}$.

Fig.~\ref{fig:objectiveregretIn} compares the objective regret for the inventory control  environment  incurred by each algorithm with respect to the number of episodes. As we see in this figure, in the initial episodes, the objective regret of DOPE grows linearly with number of episodes.  Later, the growth rate of regret changes to square root of number of episodes.  We see that this change of behavior happens after $K_0$ episodes specified by Proposition \ref{prop:Feasibility}.  Hence, the linear growth rate indeed corresponds to the duration of time in which the base policy is employed. In conclusion, the regret for DOPE algorithm depicted in Figures \ref{fig:objectiveregretIn} matches the result of Theorem \ref{thm:main-regret-theorem}.  Next, the OptPess-LP algorithm performs quite badly in terms of objective regret, as it fails to achieve $\sqrt{K}$ regret performance within the chosen number of episodes.  It thus shows the same issue of excessive pessimism observed in the other environments. Finally, we  observe that the objective regret of OptCMDP is lower than DOPE.  This behavior can be attributed to the fact that in order to perform safe exploration, DOPE includes a pessimistic penalty in the constraint \eqref{eqn:pessimistic-cost}.

\begin{figure*}[t]

    \subfigure[{\footnotesize Objective Regret}]{\label{fig:objectiveregretIn}%
      \includegraphics[width=0.31\linewidth,clip=true,trim=4mm 4mm 3.5mm 4mm]{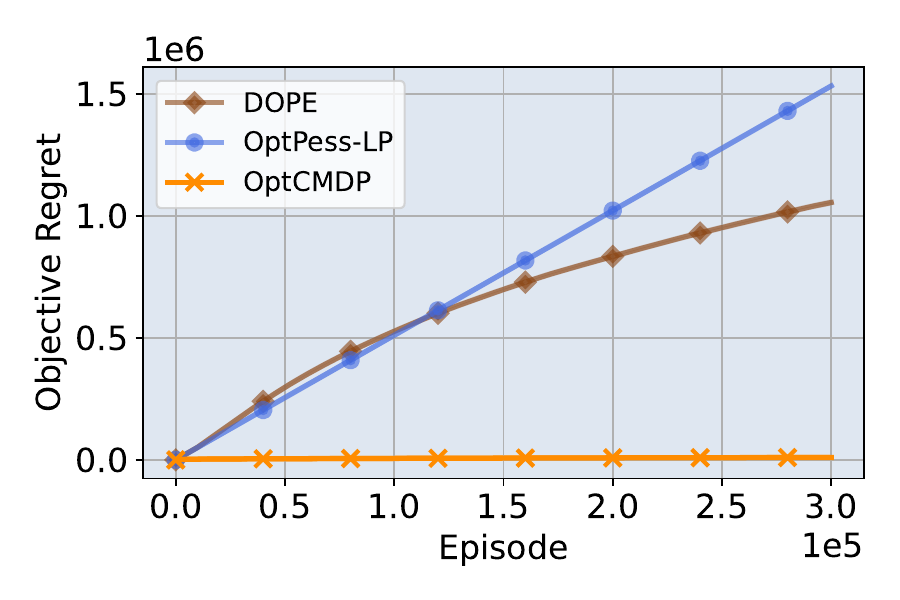}}%
    ~
    \subfigure[{\footnotesize Constraint Regret}]{\label{fig:constraint-regretIn}%
      \includegraphics[width=0.31\linewidth,clip=true,trim=4mm 4mm 3.5mm 4mm]{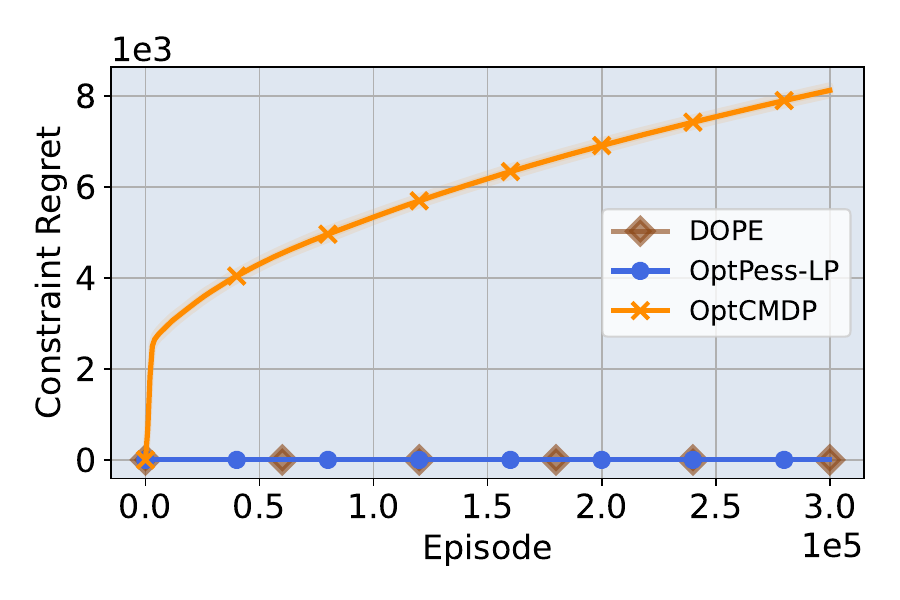}}
    ~
    \subfigure[{\footnotesize Objective Regret, varying $\bar{C}_b$}]{\label{fig:obj-regretIncb}%
      \includegraphics[width=0.31\linewidth,clip=true,trim=4mm 2mm 3.5mm 4mm]{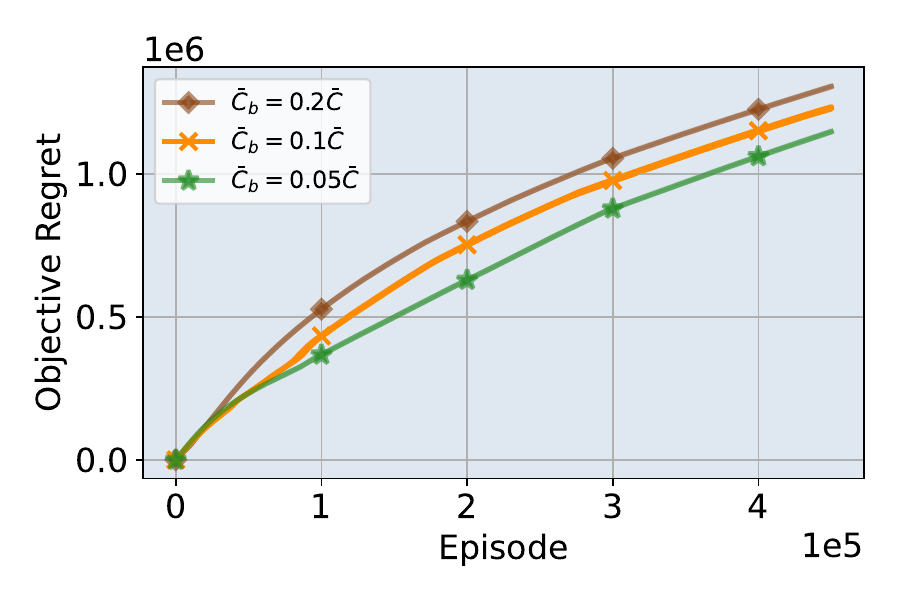}}      
  
  {\caption{Illustrating the Objective Regret and Constraint Regret for the Inventory Control Environment.} \label{fig: InventoryControlPlots}}

\end{figure*}

Fig.~\ref{fig:constraint-regretIn} compares the regret in constraint violation for DOPE, OptPess-LP and OptCMDP algorithms for the inventory control setting.  Here, we see that DOPE and OptPess-LP do not violate the constraint, while OptCMDP incurs a regret that grows sublinearly.  This figure shows that DOPE does indeed perform safe exploration as proved, while OptCMDP violates the constraints during learning.

Finally, Fig.~(\ref{fig:obj-regretIncb}) compares the optimality regret for various baseline policies.  Again, the takeaway here is that a good baseline policy is helpful, although the variation across different baseline policies is not very large.

\end{document}